\documentclass[journal,twocolumn]{IEEEtran}

\usepackage{graphicx} 
\usepackage[utf8]{inputenc} %
\usepackage[T1]{fontenc}    %
\usepackage{hyperref}       %
\usepackage{url}            %
\usepackage{booktabs}       %
\usepackage{amsfonts}       %
\usepackage{nicefrac}       %
\usepackage{microtype}      %
\usepackage{caption}
\usepackage{microtype}
\usepackage{graphicx}
\usepackage{subcaption}
\usepackage{booktabs} %
\usepackage{enumitem} 
\usepackage{bm}
\usepackage{placeins}
\usepackage{balance}
\usepackage{algorithmicx,algpseudocode,algorithm}
\usepackage{wrapfig,tikz}
\usepackage{amsmath,longtable,fancyhdr,booktabs,multirow,graphicx,float}
\usepackage{adjustbox, bigstrut, tabularx, multirow, makecell, diagbox}
\usepackage{amssymb,xcolor,amsthm}
\usepackage{color}
\usepackage{colortbl}
\usepackage{theoremref}
\usepackage{stmaryrd}
\usepackage{xcolor}
\usepackage{amsthm}

\newtheorem{theorem}{Theorem}[section]
\newtheorem{lemma}[theorem]{Lemma}

\newtheorem{assumption}[theorem]{Assumption}
\theoremstyle{remark}
\newtheorem*{remark}{Remark}

\ifCLASSINFOpdf
\else
\fi

\begin{document}

\title{Conditional Score Learning for Quickest Change Detection in Markov Transition Kernels}

\author{
      Wuxia~Chen, Taposh~Banerjee, Vahid~Tarokh
}

\maketitle

\IEEEpeerreviewmaketitle

\begin{abstract}
We address the problem of quickest change detection in Markov processes with unknown transition kernels.
The key idea is to learn the conditional score $\nabla_{\mathbf{y}} \log p(\mathbf{y}|\mathbf{x})$ directly from sample pairs $( \mathbf{x},\mathbf{y})$, where both $\mathbf{x}$ and $\mathbf{y}$ are high-dimensional data generated by the same transition kernel. In this way, we avoid explicit likelihood evaluation and provide a practical way to learn the transition dynamics. Based on this estimation, we develop a score-based CUSUM procedure that uses conditional Hyvärinen score differences to detect changes in the kernel. To ensure bounded increments, we propose a truncated version of the statistic. 
With Hoeffding’s inequality for uniformly ergodic Markov processes, we prove exponential lower bounds on the mean time to false alarm. We also prove asymptotic upper bounds on detection delay.  These results give both theoretical guarantees and practical feasibility for score-based detection in high-dimensional Markov models.
\end{abstract}


\section{Introduction}
Detecting changes in sequential data is a fundamental problem with broad applications, including biology~\cite{siegmund2013change}, quality control~\cite{hawkins2003changepoint}, and power systems~\cite{chen2015quickest}. In statistics, this problem is formulated as the problem of quickest change detection (QCD), where the goal is to detect distributional changes as quickly as possible while controlling false alarms~\cite{shiryaev1963optimum}.  

A classical solution is the Cumulative Sum (CUSUM) algorithm~\cite{page1954continuous}. CUSUM is recursive, easy to implement, and optimal in the i.i.d. setting under Lorden’s and Pollak’s criteria~\cite{lorden1971procedures, moustakides1986optimal, pollak1985optimal}. For comprehensive overviews, see~\cite{veeravalli2014quickest, tartakovsky2014sequential, tartakovsky2019sequential, basseville1993detection}.

However, CUSUM and related optimal methods require explicit knowledge of the pre- and post-change distributions. In high-dimensional applications, such as energy-based models~\cite{lecun2006tutorial} or score-based generative models~\cite{song2019generative}, evaluating likelihoods is often infeasible.  

Several recent works have advanced score-based approaches for quickest change detection and hypothesis testing. These methods replace log-likelihood ratios with differences of Hyvärinen scores and analyze performance using Fisher divergence in place of the Kullback–Leibler divergence. Examples include score-based QCD for unnormalized models~\cite{Wuetal_IT_2024,Banerjee_SQA_2024}, Bayesian formulations for score-based models~\cite{banerjee2024bayesian}, large deviation analysis of score-based hypothesis testing~\cite{diao2024large}, and robust extensions based on least-favorable scores~\cite{moushegian2025robust}. While these papers demonstrate the strength of Fisher divergence as the natural information measure in score-based settings, they all rely on i.i.d. data to learn the score.

In contrast, real-world processes often exhibit dependence, such as Markovian dynamics, where the i.i.d. assumption does not hold. Our work addresses this gap by extending score-based change detection to dependent data. We introduce conditional score learning together with a conditional Fisher divergence, providing a new framework for analyzing changes in Markov processes. The key idea is to learn the conditional score  
\[
\nabla_{\mathbf{y}} \log p(\mathbf{y}|\mathbf{x})
\]  
directly from paired samples $(\mathbf{x},\mathbf{y})$, where both $\mathbf{x}$ and $\mathbf{y}$ are high-dimensional vectors generated by the same transition kernel. A neural network trained via score matching serves as the conditional score estimator.  

Using this estimator, we propose a conditional score-based CUSUM procedure that generalizes existing score-based algorithms from i.i.d. data to dependent data. To ensure bounded increments and numerical stability, we introduce a truncated version of the statistic. 
We derive a lower bound on the mean time to a false alarm for uniformly ergodic Markov processes. We also obtain an asymptotic upper bound on the average detection delay. Our proof techniques are inspired by the techniques used in ~\cite{chen2022change} and \cite{xian2016online}. Finally, to demonstrate the effectiveness of the proposed algorithm, we apply it to CMU motion capture data ~\cite{CMU_MoCap} to detect a change in the type of activity being performed (running, playing, jumping, etc). We also apply our algorithm to a simulated Markov process where there is a change in the Gaussian transition kernel.

The main contributions of this paper are:  
\begin{itemize}
    \item A framework for conditional score learning that estimates $\nabla_{\mathbf{y}} \log p(\mathbf{y}|\mathbf{x})$ from paired samples of the same transition kernel, enabling score-based analysis in Markov processes.  
    \item A conditional score-based CUSUM procedure for dependent data, with a truncated version ensuring bounded increments and stable implementation.  
    \item Theoretical performance guarantees, including an exponential lower bound on the mean time to false alarm and an asymptotic upper bound on detection delay, derived using Hoeffding's inequalities.  
    \item Application to CMU motion capture data and simulated Markov process data. 
\end{itemize}

The rest of the paper is organized as follows. In Section~\ref{sec:conditional_score_def}, we present the conditional score learning framework and its properties.  In Section~\ref{sec:score_learning}, we introduce a score matching method for conditional score learning via a neural network. In Section~\ref{sec:model_formulation}, we formulate the change detection problem in Markov processes. In Section~\ref{sec:algorithm_trunction}, we introduce the conditional score-based CUSUM algorithm with truncation. In Sections~\ref{sec:FA} and~\ref{sec:delay}, we analyze false alarm and delay performance. Finally, in Section~\ref{sec:numerical}, we apply the proposed algorithm to real and simulated data.

\section{Conditional Score Definition and Properties}
\label{sec:conditional_score_def}

To formalize our framework, we begin with a general pair of random vectors 
$(\mathbf{X},\mathbf{Y})$ defined on measurable spaces $\mathcal{X}$ and $\mathcal{Y}$ with  joint density $p(\mathbf{x}, \mathbf{y})$. The corresponding conditional distribution of $\mathbf{Y}$ 
given $\mathbf{X} = \mathbf{x}$ is denoted by $p(\mathbf{y}|\mathbf{x})$. While classical score-based approaches~\cite{hyvarinen2005estimation} typically focus on the gradient of an 
unconditional density $\nabla_{\mathbf{y}} \log p(\mathbf{y})$, our interest lies in conditional 
distributions, which arise naturally in dependent-data settings.

We introduce the notion of a conditional score, defined as the gradient 
of the conditional log-density with respect to $y$:
\begin{equation}
    \nabla_{\mathbf{y}} \log p(\mathbf{y}|\mathbf{x}).
\end{equation}
The conditional score captures the sensitivity of $\log p(\mathbf{y}|\mathbf{x})$ to changes 
in $\mathbf{y}$, with $\mathbf{x}$ held fixed. This generalizes the classical score function 
used in the i.i.d. setting~\cite{Wuetal_IT_2024, song2021scorebased} to models 
with conditional structure and temporal or spatial dependencies.

Inspired by the original formulation of the Hyvärinen score for unnormalized models~\cite{hyvarinen2005estimation}, we define a version suited to the conditional density  $p(\mathbf{y}|\mathbf{x})$. This conditional Hyvärinen score is given by
\begin{equation}
     S_H(\mathbf{y},\mathbf{x}; p) \triangleq 
   \frac{1}{2} \left\| \nabla_{\mathbf{y}} \log p(\mathbf{y}|\mathbf{x}) \right\|^2  
   + \Delta_{\mathbf{y}} \log p(\mathbf{y}|\mathbf{x}), 
\label{eq:conditional_Hyvarinen}
\end{equation}
where  $\Delta_y $ denotes the Laplacian with respect to  $y$. It aggregates the squared norm and divergence of the conditional score into a scalar quantity.

For high-dimensional data, this quantity can be efficiently learned from i.i.d. samples  $(\mathbf{x}_1, \mathbf{y}_1), \ldots (\mathbf{x}_n, \mathbf{y}_n), \ldots $. Moreover, when  $\mathbf{Y} $ is independent of  $\mathbf{X}$, our definition reduces to the classical Hyvärinen score for marginal densities, thereby recovering the i.i.d. setting as a special case.

\subsection{Hyvärinen Score Difference Between Conditional Models}
We now analyze how the conditional Hyvärinen scores differ under two models and study the implications of this difference. 

Let  $p(\mathbf{y}|\mathbf{x})$  and  $q(\mathbf{y}|\mathbf{x})$  be two conditional densities, associated with joint distributions $ p(\mathbf{x},\mathbf{y}) $ and $ q(\mathbf{x},\mathbf{y}) $, respectively. We define the difference in conditional Hyvärinen scores between the two models at a given pair $ (\mathbf{x},\mathbf{y}) $ as
\begin{equation}
s(\mathbf{y}, \mathbf{x}; p, q) \triangleq S_H(\mathbf{y}, \mathbf{x}; p) - S_H(\mathbf{y}, \mathbf{x}; q), \label{eq:hyvarinen_difference_general}
\end{equation}
where $ S_H(\mathbf{y}, \mathbf{x}; p) $ denotes the conditional Hyvärinen score defined in \eqref{eq:conditional_Hyvarinen}. When the context is clear, we will write $ s(\mathbf{y}, \mathbf{x}) $ for simplicity.

This score difference quantifies how much the local structure of the conditional density $ p(\mathbf{y}|\mathbf{x}) $ deviates from that of $ q(\mathbf{y}|\mathbf{x}) $, and serves as the core statistic in our change detection procedure.

\subsection{Drift Characterization via Conditional Fisher Divergence}
\label{sec:drifts}

To characterize the behavior of $ s(\mathbf{y}, \mathbf{x}) $, we consider its expectation when $ (\mathbf{x}, \mathbf{y}) \sim p $. This expected score difference acts as a drift signal: if $ p $ is the data-generating model, then $ \mathbb{E}_{(\mathbf{x}, \mathbf{y})\sim p}[s(\mathbf{y}, \mathbf{x})] < 0 $; if $ q $ is the data-generating model, the drift becomes positive.

We define the conditional Fisher divergence between $ p(\mathbf{y}|\mathbf{x}) $ and $ q(\mathbf{y}|\mathbf{x}) $ as
\begin{align}
D_F&(p \| q \mid \mathbf{x}) 
\triangleq \mathbb{E}_{\mathbf{y} \sim p(\cdot | \mathbf{x})} 
\left\| \nabla_{\mathbf{y}} \log p(\mathbf{y}|\mathbf{x}) - \nabla_{\mathbf{y}} \log q(\mathbf{y}|\mathbf{x}) \right\|^2  \label{eq:conditional_fisher_divergence}\\
&= \mathbb{E}_{\mathbf{y} \sim p(\cdot | \mathbf{x})} \left[ 
\frac{1}{2} \left\| \nabla_{\mathbf{y}} \log p(\mathbf{y}|\mathbf{x}) \right\|^2 
+ S_H(y, x; q) 
\right]. \label{eq:key_identity_general}
\end{align}
Equation~\eqref{eq:conditional_fisher_divergence} is always nonnegative and vanishes if and only if the conditional score functions under $ p $ and $ q $ coincide almost surely. The identity in~\eqref{eq:key_identity_general} expresses the conditional Fisher divergence in terms of the Hyvärinen score under model $ q $, and holds under mild regularity conditions~\cite{hyvarinen2005estimation}.

Substituting the definition of the score difference $ s(\mathbf{y}, \mathbf{x}) $, we obtain
\begin{equation}
\mathbb{E}_{\mathbf{y} \sim p(\cdot | \mathbf{x})} \left[ s(\mathbf{y}, \mathbf{x}) \right] 
= - D_F(p \| q \mid \mathbf{x}). \label{eq:drift_given_x}
\end{equation}
Assume that the conditional models $ p(\mathbf{y}|\mathbf{x}) $ and $ q(\mathbf{y}|\mathbf{x}) $ admit stationary marginal distributions $ \pi_p $ and $ \pi_q $, respectively. When the data are generated according to the true model $ p(\mathbf{y}|\mathbf{x}) $, we evaluate the expected score difference under the joint distribution $ p(\mathbf{x},\mathbf{y}) = \pi_p(\mathbf{x}) p(\mathbf{y}|\mathbf{x}) $. Taking expectation over $ \mathbf{x} \sim \pi_p $, we obtain the stationary drift:
\begin{equation}
\mathbb{E}_{(\mathbf{x}, \mathbf{y}) \sim p} \left[ s(\mathbf{y}, \mathbf{x}) \right] 
= - \mathbb{E}_{\mathbf{x} \sim \pi_p} \left[ D_F(p \| q \mid \mathbf{x}) \right] < 0.
\label{eq:negative_drift_general}
\end{equation}
Similarly, when the data are generated from the alternative model $ q(\mathbf{y}|\mathbf{x}) $, and the marginal distribution of $ \mathbf{x} $ is $ \pi_q $, the expected score difference becomes
\begin{equation}
\mathbb{E}_{(\mathbf{x}, \mathbf{y}) \sim q} \left[ s(\mathbf{y}, \mathbf{x}) \right] 
=  \mathbb{E}_{\mathbf{x} \sim \pi_q} \left[ D_F(q \| p \mid \mathbf{x}) \right] > 0.
\label{eq:positive_drift_general}
\end{equation}
These stationary drifts \eqref{eq:negative_drift_general} and \eqref{eq:positive_drift_general} provide a consistent statistical signature that distinguishes the true model from the alternative. The sign of the expected score difference serves as the foundation for the change detection algorithm introduced in subsequent sections.

\section{Conditional Score Learning}
\label{sec:score_learning}

In this section, we address the problem of estimating the conditional score function $ \nabla_{\mathbf{y}} \log p(\mathbf{y}|\mathbf{x}) $, where the conditional density $ p(\mathbf{y}|\mathbf{x}) $ is unknown. Our goal is to learn a parametric model that approximates this conditional score using only samples of the form $ (\mathbf{x}, \mathbf{y}) \sim p $. This estimation problem arises naturally in high-dimensional settings where computing or even evaluating the conditional density is intractable.

Although estimating the conditional score is generally a non-parametric problem, we show that it admits a tractable and consistent formulation. In particular, similar to~\cite{vincent2011connection, song2019generative, Wuetal_IT_2024}, the learning objective can be expressed entirely in terms of model-defined derivatives and expectations under the empirical data distribution, without requiring access to $ p(\mathbf{y}|\mathbf{x}) $ itself. This makes conditional score learning possible using only observed samples $ (\mathbf{x}, \mathbf{y}) $. 

\subsection{Approximating $ \nabla_{\mathbf{y}} \log p(\mathbf{y} | \mathbf{x}) $  via Score Matching}

We propose to estimate the conditional score $ \nabla_{\mathbf{y}} \log p(\mathbf{y} | \mathbf{x})$ using a parameterized neural network
\[
\boldsymbol{\psi}(\mathbf{y}, \mathbf{x}; \boldsymbol{\theta}) 
= \begin{pmatrix}
\psi_1(\mathbf{y}, \mathbf{x}; \boldsymbol{\theta}) \\
\vdots \\
\psi_d(\mathbf{y}, \mathbf{x}; \boldsymbol{\theta})
\end{pmatrix},
\]
where $ \boldsymbol{\theta} $ denotes the model parameters. Since direct access to the conditional density is not available, we leverage the identity
\begin{align}
\nabla_{\mathbf{y}} \log p(\mathbf{y} | \mathbf{x}) 
&= \nabla_{\mathbf{y}} \log p(\mathbf{x}, \mathbf{y}) - \nabla_{\mathbf{y}} \log p(\mathbf{x}) \notag \\
&= \nabla_{\mathbf{y}} \log p(\mathbf{x}, \mathbf{y}),
\label{eq:score_joint_identity}
\end{align}
where we treat $ \mathbf{x} $ as fixed during differentiation. This identity allows us to treat the conditional score as the gradient of the joint log-density with respect to $ \mathbf{y} $, which can be estimated via score matching.

We define the loss function
\begin{align}
J(\boldsymbol{\theta}) \triangleq 
\mathbb{E}_{\mathbf{x}, \mathbf{y} \sim p} 
\left[
\frac{1}{2} \left\| 
\boldsymbol{\psi}(\mathbf{y}, \mathbf{x}; \boldsymbol{\theta}) 
- \nabla_{\mathbf{y}} \log p(\mathbf{y} | \mathbf{x}) 
\right\|^2
\right], \label{eq:loss_conditional}
\end{align}
which measures the expected squared error between the model output and the true conditional score. If we obtain parameters $\boldsymbol{\theta}^*$ that minimize the training objective $J(\boldsymbol{\theta})$ with
\[
\boldsymbol{\theta}^* = \arg \min_{\boldsymbol{\theta}} J(\boldsymbol{\theta}),
\]
then the learned network $\boldsymbol{\psi}(\mathbf{y}, \mathbf{x}; \boldsymbol{\theta}^*)$ recovers the conditional score $\nabla_{\mathbf{y}} \log p(\mathbf{y} | \mathbf{x})$.

Expanding the objective \eqref{eq:loss_conditional}, we obtain
\begin{align}
J(\boldsymbol{\theta}) = 
\mathbb{E}_{\mathbf{x}, \mathbf{y} \sim p} &\left[
\frac{1}{2} \left\| 
\boldsymbol{\psi}(\mathbf{y}, \mathbf{x}; \boldsymbol{\theta}) 
\right\|^2
+ \frac{1}{2} \left\| 
\nabla_{\mathbf{y}} \log p(\mathbf{y} | \mathbf{x}) 
\right\|^2 \right. \notag \\
& \quad \left.
- \boldsymbol{\psi}(\mathbf{y}, \mathbf{x}; \boldsymbol{\theta})^\top 
\nabla_{\mathbf{y}} \log p(\mathbf{y} | \mathbf{x}) 
\right]. \label{eq:loss_expanded}
\end{align}
For \eqref{eq:loss_expanded}, the first term can be expressed as 
\[
\frac{1}{2} \|\boldsymbol{\psi}(\mathbf{y}, \mathbf{x};\boldsymbol{\theta})\|^2  =  \frac{1}{2}\sum_{i=1}^d   \psi_i^2(\mathbf{y}, \mathbf{x};\boldsymbol{\theta}).
\]
The second term  $\frac{1}{2}  \| \nabla_{\mathbf{y}} \log p(\mathbf{y} |\mathbf{x}) \|^2 $ is not a function of $\boldsymbol{\theta}$. 

Define $y_i$ as the $i$-th component of $\mathbf{y}$, and let 
\[
\mathbf{y}_{-i} \triangleq (y_1, \dots, y_{i-1}, y_{i+1}, \dots, y_d) 
\]
denote the collection of all components of $\mathbf{y}$ except $y_i$. 
Combining with \eqref{eq:score_joint_identity},  the third term can be written as 
\begin{equation}
\begin{aligned}
& \mathbb{E}_{\mathbf{x}, \mathbf{y}} \bigg( - \boldsymbol{\psi}(\mathbf{y}, \mathbf{x};\boldsymbol{\theta})^T  \nabla_{\mathbf{y}} \log p(\mathbf{y} |\mathbf{x})  \bigg)   =  \\
&\sum_{i=1}^d   \int_{\mathbf{x}} \int_{\mathbf{y}} -    \psi_i(\mathbf{y}, \mathbf{x};\boldsymbol{\theta}) 
\frac{\partial \log p(\mathbf{y}, \mathbf{x})}{\partial y_i} 
d\mathbf{y} d\mathbf{x}.    
\end{aligned}  \label{eq:3rd_term}
\end{equation}
Fix $i \in \{1, \dots, d\}$, the inside integration of \eqref{eq:3rd_term} is
\begin{align}
    & \int_{\mathbf{y}}   -   \psi_i(\mathbf{y}, \mathbf{x};\boldsymbol{\theta}) 
  \frac{\partial \log p(\mathbf{y}, \mathbf{x})}{\partial y_i}  d\mathbf{y}  \nonumber\\
    & =  \int_{\mathbf{y}_{-i}} \int_{y_i}    -   \psi_i(\mathbf{y}, \mathbf{x};\boldsymbol{\theta}) 
    \frac{\partial  p(\mathbf{y}, \mathbf{x})}{\partial y_i}  dy_i  d\mathbf{y}_{-i}   \nonumber\\ 
    & = \int_{\mathbf{y}_{-i}} \bigg( -\psi_i(\mathbf{y}, \mathbf{x};\boldsymbol{\theta})  p(\mathbf{y}, \mathbf{x})\bigg|_{y_i= -\infty}^{y_i = \infty}  \nonumber\\
    & \quad + \int_{y_i}  p(\mathbf{y}, \mathbf{x}) \frac{\partial  \psi_i(\mathbf{y}, \mathbf{x};\boldsymbol{\theta})  }{ \partial y_i}  dy_i  \bigg) d\mathbf{y}_{-i}  \label{eq:int_byparts} \\ 
   & =   \int_{\mathbf{y}}  p(\mathbf{y}, \mathbf{x}) \frac{\partial  \psi_i(\mathbf{y}, \mathbf{x};\boldsymbol{\theta})  }{ \partial y_i} d\mathbf{y} \label{eq:regularity}.
\end{align}
Equation~\eqref{eq:int_byparts} follows from integration by parts. 
Equation~\eqref{eq:regularity} holds, because of the regularity conditions\footnote{The regularity conditions are: (i)  transition kernel $ p(\mathbf{y} | \mathbf{x})$ is differentiable. (ii) Both expectations  $\mathbb{E}_{\mathbf{x}, \mathbf{y}} \left\| \boldsymbol{\psi}(\mathbf{y}, \mathbf{x}; \boldsymbol{\theta}) \right\|^2 $ and 
$\mathbb{E}_{\mathbf{x}, \mathbf{y}} \left\|  \nabla_{\mathbf{y}} \log p(\mathbf{y} |\mathbf{x}) \right\|^2  $ are finite for any $ \boldsymbol{\theta} $. (iii) The boundary term
$p(\mathbf{y} | \mathbf{x})  \boldsymbol{\psi}(\mathbf{y}, \mathbf{x}; \boldsymbol{\theta}) \to \boldsymbol{0}$ as $ \| \mathbf{y} \| \to \infty$, 
for all $ \mathbf{x} $ and any $ \boldsymbol{\theta} $.
}~\cite{hyvarinen2005estimation}.

Substituting Equation~\eqref{eq:regularity} into Equation~\eqref{eq:3rd_term}, we obtain
\begin{align*}
& \sum_{i=1}^d   \int_{\mathbf{x}} \int_{\mathbf{y}}  p(\mathbf{y}, \mathbf{x}) \frac{\partial  \psi_i(\mathbf{y}, \mathbf{x};\boldsymbol{\theta})  }{ \partial y_i} d\mathbf{y} d\mathbf{x}  \\
&= \mathbb{E}_{\mathbf{x}, \mathbf{y}} \bigg(  \sum_{i=1}^d   \frac{\partial  \psi_i(\mathbf{y}, \mathbf{x};\boldsymbol{\theta})  }{ \partial y_i} \bigg).
\end{align*}
Accordingly, we define the surrogate objective function 
\begin{align}
 \widetilde{J}(\boldsymbol{\theta}) & \triangleq \mathbb{E}_{\mathbf{x}, \mathbf{y}} \bigg(   \sum_{i=1}^d \frac{1}{2} \psi_i^2(\mathbf{y}, \mathbf{x};\boldsymbol{\theta})  +  \frac{\partial  \psi_i(\mathbf{y}, \mathbf{x};\boldsymbol{\theta})  }{ \partial y_i} \bigg) .    
 \label{eq:surrogate_loss}
\end{align}
Then the original loss function in \eqref{eq:loss_conditional} can be written as
\begin{align*}
   J(\boldsymbol{\theta}) =  \widetilde{J}(\boldsymbol{\theta})  + \text{constant}.
\end{align*}
This surrogate objective is advantageous in practice, as it can be directly minimized using gradient-based optimization without requiring access to the transition density $p(\mathbf{y}| \mathbf{x})$. Since $\widetilde{J}(\boldsymbol{\theta})$ and $J(\boldsymbol{\theta})$ differ only by a constant, both share the same minimizer. 

Consequently, at convergence, the learned score network $\boldsymbol{\psi}(\mathbf{y}, \mathbf{x}; \boldsymbol{\theta}) $
provides an accurate estimator  of the conditional score function $\nabla_{\mathbf{y}} \log p(\mathbf{y}| \mathbf{x})$.
In practice, the score network can be implemented as a convolutional U-Net~\cite{ronneberger2015u, song2021scorebased} and trained using a denoising score matching objective adapted for conditional distributions~\cite{vincent2011connection, song2019generative}.

\section{Change Point Model Formulation} \label{sec:model_formulation}

Consider a stochastic process $ \{X_n\}_{n \in \mathbb{N}} $ defined on a probability space $ (\Omega, \mathcal{F}, \mathbb{P}) $, taking values in a measurable state space $ \mathbb{X}$. We assume that the process undergoes a change in its transition dynamics at an unknown change-point $ \nu \in \mathbb{N} $.

Let $ p(\cdot | x) $ and $ q(\cdot | x) $ be two transition densities defined with respect to a common reference measure $ \mu $ (e.g., Lebesgue measure) on $ \mathbb{X} $. These define the pre-change and post-change Markov transition kernels $ P $ and $ Q $, respectively, via:
\begin{align*}
  P(x, A) = \int_A p(y | x) \, d\mu(y), \\
Q(x, A)  = \int_A q(y |d x) \, d\mu(y), \\
 \text{for all measurable } A \subseteq \mathbb{X}.
\end{align*}

We assume both $ P $ and $ Q $ correspond to stationary and ergodic Markov processes with stationary distributions $ \pi_p $ and $ \pi_q $ respectively. Given the initial state $ X_0 \sim \pi_p $, the process evolves according to:
\[
X_n \mid X_{n-1} \sim 
\begin{cases}
p(\cdot \mid X_{n-1}), & n < \nu, \\
q(\cdot \mid X_{n-1}), & n \geq \nu.
\end{cases}
\]
Specifically, we assume that there are two independent Markov processes. At the time $\nu$, the observation process abruptly switches from the Markov process with kernel $p$ to the one with kernel $q$. 

The objective is to detect the change from $ P$ to $ Q $ as quickly as possible while controlling the false alarm. To evaluate the performance of a stopping time $ T $, we use two standard criteria:
\begin{align}
\mathrm{MTBFA}(T) \triangleq \mathbb{E}_\infty[T], \\
\mathrm{MD}(T) \triangleq \mathbb{E}_\nu[T - \nu \mid T > \nu],
\end{align}
where $ \mathbb{E}_\infty $ denotes expectation under the no-change hypothesis (i.e., all observations follow $ P $, the change time $\nu$ is $\infty$), and $ \mathbb{E}_\nu $ denotes expectation under the change-point model with transition occurring at $ \nu $.

\begin{assumption}[Doeblin’s Condition]
\label{thm:Doeblin’s Condition}
A Markov kernel $ P $ satisfies Doeblin’s condition if there exist a probability measure $ \phi $ on $ \mathbb{X}$, a constant $ \lambda > 0 $, and an integer $ l \in \mathbb{N}^+ $ such that
\[
P^l(x, \cdot) \geq \lambda \phi(\cdot), \quad \text{for all } x \in \mathbb{X}.
\]
\end{assumption}
\begin{remark}
Doeblin’s condition implies uniform ergodicity for aperiodic and irreducible Markov chains; see~\cite[Theorem 16.2.3 ]{meyn2012markov}. While the constants $ l $ and $ \lambda $ may differ between $ P $ and $ Q $, we assume that both kernels satisfy Doeblin’s condition with potentially distinct parameters. This assumption is critical to our analysis, as it allows the application of Hoeffding-type concentration inequalities, which are instrumental in deriving performance guarantees for change detection procedures. 
\end{remark}

\begin{lemma}
Let $ \{X_n\}_{n \in \mathbb{N}} $ be a Markov chain satisfying Doeblin’s condition. Define the second-order process $ \{B_n\}_{n \in \mathbb{N}} \triangleq \{(X_{n-1}, X_n)\}_{n \in \mathbb{N}} $. Then, the process $ \{B_n\}_{n \in \mathbb{N}} $ also satisfies Doeblin’s condition.
\end{lemma}

\begin{remark}
This lemma originally appears in~\cite[Lemma 2]{chen2022change}, to establish concentration inequalities for second-order Markov processes derived from first-order Markov chains. 
\end{remark}

\subsection{Optimal Algorithm}
When the pre- and post-change kernels  $p(\cdot | x)$  and $q(\cdot| x)$ are known, the optimal change detection rule is the CUSUM procedure~\cite{lorden1971procedures}\cite{lai1998information}
\[
\tau= \inf \bigg\{ n \geq 1 : \max_{1 \leq k \leq n} \sum_{i=k}^n 
\log \frac{q(X_i | X_{i-1})}{p(X_i |X_{i-1})} \;\; \geq b \bigg\},
\]
where $b > 0$ is a detection threshold. This stopping time is asymptotically optimal with respect to Lorden’s worst-case delay criterion~\cite{lorden1971procedures} and Pollak’s average delay criterion~\cite{pollak1985optimal}.

In practice, the optimal procedure is not directly implementable, since the exact forms of the transition kernels are typically unknown. To address this limitation, we develop a Hyvärinen score-based alternative that avoids the need for explicit likelihood evaluation. 

\section{Score-Based Change Detection in Markov Processes}
\label{sec:algorithm_trunction}
We now propose a change detection procedure specifically designed for Markov processes with continuous and high-dimensional state spaces. The method relies on the Hyvärinen score to assess local discrepancies between two candidate transition kernels.  The key idea is to accumulate evidence over time in the form of Hyvärinen score differences between the pre-change transition kernel  $P $  and post-change kernel  $Q$. A change is declared when the cumulative discrepancy becomes too large to be explained by random variation.

Let $B_n \triangleq ( X_{n-1},X_n)$ denote a pair of consecutive observations from the Markov process. As defined in equation~\eqref{eq:hyvarinen_difference_general}, we write  
\[
s(B_n) = S_H(B_n; p) - S_H(B_n; q)
\]
to denote the Hyvärinen score difference between the two kernels $P$ and $Q$ at time $n$. 

\subsection{Stopping Rule Based on Hyvärinen Score Differences}

To formalize this procedure, we define the stopping time $T(b)$ as the first time a cumulative sum of Hyvärinen score differences exceeds a fixed threshold $b$. This corresponds to a generalized CUSUM rule adapted to score-based models and dependent observations. Specifically, we define
\begin{equation}
      \Tilde{T}(b)  \triangleq \inf \{ n \geq 1 : \max_{1 \leq k \leq n}  
    \sum_{i=k}^n  s(B_i)  \geq b \},   
\label{eq:no_truncation_alg}
\end{equation}
As shown in Section~\ref{sec:drifts}, the drift of $\mathbb{E}_{\infty}[s(B_i) ]$
shifts from negative before the change to positive after it, enabling reliable detection of changes in the Markov dynamics.

\subsection{Truncated Stopping Rule and Drift Preservation}
In practice, the increments $s(B_i)$ can be unbounded, especially in continuous-state models. 
To address this, we introduce a truncated version of the stopping rule that limits the influence of extreme values. Specifically, we define
\begin{equation}
      T(b) \triangleq \inf \bigg\{ n \geq 1 : 
      \max_{1 \leq k \leq n}  
    \sum_{i=k}^n \varphi\Big( s(B_i) \Big) \geq b \bigg\},
    \label{eq:truncated_algorithm}
\end{equation}
where $b > 0$ is the detection threshold controlling the trade-off between false alarms and detection delay, and
$ \varphi:\mathbb{R} \to \mathbb{R}$ is a truncation function defined by
\[
\varphi(x) \triangleq 
\begin{cases}
x, & |x| \le M, \\
M, & x > M, \\
-M, & x < -M,
\end{cases}
\]
for some large enough $M > 0$.

Truncation serves two key roles. Theoretically, our performance analysis relies on Hoeffding’s inequality for Markov chains~\cite{glynn2002hoeffding}, which requires the increments to be bounded.   Practically, it prevents rare but extreme score values from dominating the statistic, thereby improving numerical stability during implementation.

A natural question arises: Does truncation preserve the drift property that underpins the detection mechanism? We show that, under a mild light-tail condition~\cite[Sec. 2.2]{dembo2009large}, the truncated increment still retains negative drift under the pre-change distribution $P$, provided a sufficiently large truncation level.

Let $ X $ be any real-valued random variable with $ \mathbb{E}[X] < 0 $. Define the tail probabilities
\[
p_+ \triangleq \mathbb{P}(X > M), \qquad p_- \triangleq \mathbb{P}(X < -M).
\]
We consider the truncated version $ \varphi(X) $ and examine its expectation. A direct decomposition yields
\begin{align}
\mathbb{E}[\varphi(X)] 
&= \mathbb{E}[X \cdot \mathbf{1}_{\{|X| \le M\}}] + M p_+ - M p_- \nonumber \\
&= \mathbb{E}[X] + \left(  M p_+-\mathbb{E}[X \cdot \mathbf{1}_{\{X > M\}}] \right) \nonumber \\
   &- \left( \mathbb{E}[X \cdot \mathbf{1}_{\{X < -M\}}] + M p_- \right). \label{eq:E_truncate}
\end{align}
In \eqref{eq:E_truncate}, the first term, $\mathbb{E}[X] $, is strictly negative by assumption. The second term is nonpositive. We now bound the third term by leveraging the light-tail condition. Specifically, suppose
\[
\exists \lambda > 0 \quad \text{such that} \quad \mathbb{E}[e^{\lambda |X|}] < \infty.
\]
This implies exponential decay of the lower tail: there exists a constant $ C > 0 $ such that
\[
\mathbb{P}(X < x) \le C e^{\lambda x}, \quad \forall x < 0.
\]
Using this bound, we estimate the third term as follows:
\begin{align}
& - \left( \mathbb{E}[X \cdot \mathbf{1}_{\{X < -M\}}] + M p_- \right) 
= - \int_{-\infty}^{-M} x f(x)\,dx - M F(-M) \nonumber \\
& \qquad \qquad  \qquad \qquad\le \int_{-\infty}^{-M} F(x)\,dx 
\le \frac{C}{\lambda} e^{-\lambda M},\label{eq:left_tail_bound}
\end{align}
where  $f(\cdot) $ and  $F(\cdot )$  denote the density and distribution function of $X$ , respectively.

The final bound in \eqref{eq:left_tail_bound} shows that the contribution of the left tail decays exponentially fast. Therefore, for any $ \epsilon > 0 $, one can choose $ M $ sufficiently large such that
\[
- \left( \mathbb{E}[X \cdot \mathbf{1}_{\{X < -M\}}] + M p_- \right) \le \epsilon.
\]
Substituting this bound into \eqref{eq:E_truncate} gives
\[
\mathbb{E}[\varphi(X)] \le \mathbb{E}[X] - \left( \mathbb{E}[X \cdot \mathbf{1}_{\{X > M\}}] - M p_+ \right) + \epsilon.
\]
 Hence, for large enough  $M$, the truncated expectation satisfies
\[
\mathbb{E}[\varphi(X)] < 0.
\]

\begin{remark}
We have shown that truncation preserves the negative drift under the pre-change distribution, which ensures reliable false-alarm control. 
By a symmetric argument, the positive drift is also maintained under the post-change distribution provided the score differences $s(B_i)$ have exponentially decaying tails. 
Thus, truncation both enables the use of bounded-increment concentration tools such as Hoeffding’s inequality for performance guarantees and improves numerical stability by limiting the influence of rare, extreme scores.
\end{remark}

\section{Analysis of False Alarms}
\label{sec:FA}
In this section, we derive an exponential lower bound on the mean time to false alarm under the null hypothesis. Following the framework of~\cite{chen2022change}, our approach is based on a concentration inequality for uniformly ergodic Markov chains applied to the Hyvärinen score sequence. This bound provides theoretical control on the rate of false alarms when the detection threshold is large.

\subsection{Hoeffding’s Inequality for Markov Chains}
We begin by stating a Hoeffding-type inequality for uniformly ergodic Markov chains, adapted from~\cite{glynn2002hoeffding,xian2016online}.

\begin{lemma}[Hoeffding's Inequality for Uniformly Ergodic Chains,~\cite{glynn2002hoeffding,xian2016online}]
\label{lem:Hoeffding}
Let $\{X_t\}$ be a uniformly ergodic Markov chain and let $f:\mathbb{X}\to \mathbb{R}$ be a bounded measurable function, and  $\|f\| = \sup_x |f(x) |< \infty$. Define
\[
S_n = \sum_{t=1}^n f(X_t), 
\quad \mu_f \triangleq \frac{2(l+1)\|f\|}{\lambda},
\]
where $(l,\lambda)$ are the constants satisfying Doeblin's condition. Then for any $\epsilon > 0$ and $n > \mu_f/\epsilon$, 
\[
\mathbb{P}\left( |S_n - \mathbb{E}[S_n]| \geq n\epsilon \right) 
\leq 2\exp\left(- \tfrac{2 (n\epsilon - \mu_f)^2}{n \mu_f^2}\right).
\]
\end{lemma}

\subsection{Exponential Lower Bound on Time to False Alarm}
We now apply Lemma~\ref{lem:Hoeffding} to the cumulative Hyvärinen score, which depends on consecutive pairs of observations. To account for this dependence structure, we lift the process to a bivariate state space by  $B_t=(X_{t-1}, X_t)$.

\begin{theorem}[Exponential Lower Bound on Time to False Alarm]
\label{thm:false_alarm_lower_bound}
Let $\{X_t\}$ be a uniformly ergodic Markov chain with transition kernel $P$, and let $T(b)$ be the stopping time with truncation, 
\[
T(b) = \inf \left\{ n \ge 1 : \max_{1 \le k \le n} \sum_{t=k}^n \varphi(s(X_t, X_{t-1}) )\ge b \right\}.
\]
Suppose for any Markov transition pair $(x, x')$ under transition kernel $P$, the Hyvärinen score function $s(x, x')$ satisfies
\begin{align*}
    & \|\varphi\|\triangleq \sup_{(x, x')} |\varphi(s(x, x')| < \infty, \\
    & \text{ and }\quad \mathbb{E}_\infty[\varphi(s(x, x')] = -\delta < 0
\end{align*}
Define the bivariate process $B_t = (X_t, X_{t-1})$, which evolves under the product kernel $\tilde{P} = P \otimes P$ and inherits uniform ergodicity from $\{X_t\}$. Let $(\tilde{l}, \tilde{\lambda})$ be the corresponding constants satisfying Doeblin’s condition. Then define
\begin{equation}
 \mu \triangleq \frac{2(\tilde{l}+1)\|\varphi\|}{\tilde{\lambda}}.
 \label{eq:mu_def}
\end{equation}

For any threshold $b > \mu$, the expected time to false alarm satisfies
\[
\mathbb{E}_\infty[T(b)] \ge \frac{2\sqrt{2}}{3} \exp\left( \frac{4\delta(b - \mu)}{\mu^2} \right).
\]

\end{theorem}

\begin{proof}
To simplify notation, we define
\[
S_{k:n} \triangleq \sum_{i=k}^n \psi( s(B_i)).
\]
The proof proceeds by expanding the expectation as
\begin{align}
  & E_{\infty}[T(b)]  = \sum_{N=1}^\infty \mathbb{P}_{\infty}(T(b) \ge N)   \nonumber\\
 & = 1 + \sum_{N=1}^\infty \bigg( 1- \mathbb{P}_{\infty}(T(b) \le N) \bigg)   \nonumber\\
 & = 1+ \sum_{N=1}^{\infty} \left( 1 - \mathbb{P}_{\infty} \bigg( \bigcup_{n=1}^N \{ T(b) = n \} \bigg) \right)  \nonumber\\
 & \ge 1+ \sum_{N=1}^L \left( 1 - \mathbb{P}_{\infty} \bigg( \bigcup_{n=1}^N \bigcup_{k=1}^{n} \{ S_{k:n} \ge b \} \bigg) \right) \label{eq:subset}\\
 & \ge 1+\sum_{N=1}^L \left( 1 - \sum_{n=1}^N \sum_{k=1}^{n} \mathbb{P}_{\infty} \{ S_{k:n} \ge b \} \right).  \label{eq:star}
\end{align}
Equation~\eqref{eq:subset} holds because the definition of stopping time in the truncated algorithm ~\eqref{eq:truncated_algorithm}. Equation~\eqref{eq:star} holds because of the union bound. 

Then we bound $ \mathbb{P}_{\infty} \{ S_{k:n} \ge b \}$ by applying  Hoeffding's inequality. Here we treat  function $\varphi \circ s (\cdot )$ as the function $f(\cdot )$,  and treat  $\mu = \frac{2(\tilde{l}+1)\|\varphi\|}{\tilde{\lambda}}$ as the $\mu_f$ in Lemma~\ref{lem:Hoeffding}. Under the condition  $b>\mu$, we have 
\begin{align}
    & \mathbb{P}_{\infty} \{ S_{k:n} \ge b \}  =\mathbb{P}_{\infty} \left(\sum_{t=k}^n \psi( s(B_t)) \ge b\right)  \nonumber \\
    & = \mathbb{P}_{\infty} \left(\sum_{t=k}^n (\psi( s(B_t)) -\mathbb{E}_\infty[\psi( s(B_t))] ) \ge b +\delta (n-k+1) \right) \nonumber \\
   &  \leq \exp\left( \frac{-2\left( b - \mu + (n-k+1)\delta \right)^2}{(n-k+1) \mu^2} \right). \label{eq:apply_hoeffding} 
\end{align}
For \eqref{eq:apply_hoeffding}, to obtain a uniform lower bound on the exponent in the tail probability, we apply a quadratic inequality that holds for all $m \in \mathbb{N^+}$ as the following, 
\begin{align}
 \frac{ (b-\mu+m\delta )^2}{m \mu^2}  
 &\geq \frac{4 (b-\mu) \delta }{\mu^2}  \label{eq:Cauchy_schwarz}.
\end{align}
This lower bound simplifies the exponential term in Hoeffding’s inequality and avoids dependence on $n$ and $k$.

Applying \eqref{eq:Cauchy_schwarz} and \eqref{eq:apply_hoeffding} to \eqref{eq:star} yields the bound below, valid for all $L \ge 1$:
\begin{align}
   & \sum_{N=1}^L \left( 1 - \sum_{n=1}^N \sum_{k=1}^{n} \mathbb{P}_{\infty} \{ S_{k:n} \ge b \} \right)  \ge  \nonumber\\
  & \sum_{N=1}^{L} \bigg( 1 - \frac{N(N+1)}{2} \exp\left( \frac{-8(b - \mu)\delta }{\mu^2} \right)\bigg)   \label{eq:tight}
  \triangleq g(L).
\end{align} 
The sum in \eqref{eq:tight} can be written compactly as a function $g(L)$, which we aim to maximize in order to obtain the tightest possible bound. To tighten this bound, we seek the largest value of $ L $ for which all terms in the sum remain non-negative. This requires
\[
1 - \frac{L(L+1)}{2} \cdot \exp\left( \frac{-8\delta(b - \mu)}{\mu^2} \right) \ge 0. 
\]
Solving for $L$, we obtain
\[
L^* = \frac{-1 + \sqrt{1 + 8 \exp\left( \frac{8\delta(b - \mu)}{\mu^2} \right)}}{2}, 
\]
which satisfies
\[
L^*(L^* + 1) = 2 \exp\left( \frac{8\delta(b - \mu)}{\mu^2} \right).
\]
The maximum of $g(L)$ is achieved at $L = L^*$, yielding the sharpest lower bound on the false alarm time.
Under  the condition $b > \mu$, applying the closed-form of $L^*$, yields the final bound
\begin{align*}
   & E_{\infty}[T(b)] = 1 + \sum_{N=1}^\infty \bigg( 1- \mathbb{P}_{\infty}(T(b)  \le N) \bigg    ) \\
       & \ge 1+ L^*  -  \exp\left( \frac{-8\delta (b - \mu)}{\mu^2} \right) \sum_{N=1}^{L^*}\frac{N(N+1)}{2}  \\
         & = 1+ \frac{2}{3}(L^* -1) \ge \frac{2\sqrt{2}}{3} \exp \left( \frac{4\delta (b - \mu)}{\mu^2} \right) .
\end{align*}
\end{proof}

\section{Analysis of Detection Delay}
\label{sec:delay}
We now analyze the expected detection delay under the post-change regime. Following the approach of~\cite{xian2016online}, we aim to upper bound $\text{MD}(T) = \mathbb{E}_\nu[T - \nu | T > \nu]$, when the data are generated from the Markov kernel $Q$. Without loss of generality, we assume the change-point $\nu = 1$, so that all observation pairs come from the post-change process. The analysis is based on a one-sided version of Hoeffding’s inequality for uniformly ergodic Markov chains.

\begin{theorem}[Detection Delay Upper Bound]
\label{thm:Delay_upper_bound}
Let $\{X_t\}$ be the Markov Process in Theorem~\ref{thm:false_alarm_lower_bound}, define 
\begin{align*}
     &I \triangleq \mathbb{E}_1[\varphi(s(B_t))] < \infty, \\
     &\text{ and }n_0 \triangleq \left\lfloor \frac{b + \mu}{I} \right\rfloor, 
\end{align*}
where $\lfloor \cdot \rfloor$ is the floor function. 
As $n_0 \to \infty$, the expected detection delay satisfies
\[
\mathbb{E}_1[T(b)] \le 1 + n_0(1 + o(1)).
\]
\end{theorem}

\begin{proof}
We begin with the standard expansion:
\begin{align}
     &\mathbb{E}_1[T(b)] = 1 + \sum_{n=1}^{\infty} \mathbb{P}_1(T(b) > n)  \nonumber\\
   & \le 1+  \sum_{n=1}^{\infty} \mathbb{P}_1(S_{1:n} < b)  \label{eq:stop_ge_n} \\
&\le 1+ n_0 + \sum_{n=n_0 + 1}^{\infty} \mathbb{P}_1 \bigg( \sum_{k=1}^n  \varphi(s(B_k))  <b \bigg).     \label{eq:Delay_sum}  
\end{align}
Equation~\eqref{eq:stop_ge_n} is from the stopping algorithm in \eqref{eq:truncated_algorithm}. 
Now we apply Hoeffding's inequality. For $n \geq n_0 +1$, we have
\begin{align}
   & \mathbb{P}_1 \bigg( \sum_{k=1}^n \varphi(s(B_k))  <b \bigg) \nonumber  \\
   &=  \mathbb{P}_1 \bigg( \sum_{k=1}^n  \varphi( s(B_k)) - n\ \mathbb{E}_1[\varphi( s(B_k))]  < b - nI  \bigg )  \nonumber\\
   & \le   \exp\bigg( \frac{ -2(nI  - b  - \mu)^2}{n \mu^2} \bigg)  \label{eq:tail_sum_1}.\\
   & \quad \quad \left(\text{Condition: } b - n\cdot \mathbb{E}_1 [s(B_k)] <0. \right)  \nonumber
\end{align}

Define $ C \triangleq  \frac{2I^2}{\mu^2}$, then as $n_0 \rightarrow \infty$,   the tail sum in \eqref{eq:tail_sum_1} is bounded as follows, 
\begin{align}
& \sum_{n=n_0 +1}^{\infty} \mathbb{P}_1 \bigg( \sum_{k=1}^n  \varphi(s(B_k))  <b \bigg) \nonumber \\
& \le \sum_{n=n_0 +1}^{\infty} \exp\bigg(    \frac{ -2  I^2\left(n  - \frac{b+\mu}{I} \right)^2}{n \mu^2} \bigg)  \label{eq:oneside_hoeffding}\\
&\le     \sum_{n=n_0 +1}^{\infty} \exp\bigg( \frac{ -C \left(n  - n_0 -1\right)^2}{n } \bigg)     \nonumber\\ 
 & =  1 + \sum_{k=1}^{\infty} \exp\bigg( \frac{ -C  k^2}{k+n_0 +1} \bigg) \label{eq:change_k_n} \\
 &\le 1+ \sum_{k=1}^{n_0 +1} \exp\bigg( \frac{ -C  k^2}{2(n_0 +1)} \bigg) +\sum_{k=n_0+2}^{\infty} \exp\bigg(  \frac{ -C k}{2 } \bigg)  \label{eq:split_k}\\
 &  \le  1+\int_{0}^{n_0 +1} \exp\bigg( \frac{-C x^2}{2(n_0+1)}  \bigg) dx + \int_{n_0 +1}^\infty \exp \bigg(  \frac{-C x}{2} \bigg) dx     \label{eq:continuation} \\
 &\le 1+  \sqrt{\frac{2 \pi(n_0 +1 )}{C} } + \frac{2}{C}\exp\bigg( -\frac{C(n_0+1)}{2}  \bigg).  \nonumber
\end{align}
Inequality \eqref{eq:oneside_hoeffding} follows from Hoeffding’s inequality. Equality \eqref{eq:change_k_n} holds by a change of variable $n-n_0-1=k$. Inequality \eqref{eq:split_k} uses the fact that when $k \le n_0 +1$, we have $  \frac{k^2}{k+n_0 +1} \ge \frac{k^2}{2(n_0+1) }$, and when $k>n_0 +1$, we have $  \frac{k^2}{k+n_0+1 } \ge \frac{k^2}{2k }$. Finally, \eqref{eq:continuation} bounds the discrete sums by integrals.

Indeed, we have
\[
\lim_{n_0 \rightarrow \infty}  \frac{ 1+ \sqrt{\frac{ 2\pi(n_0 +1) }{C} } + \frac{2}{C}\exp\bigg( -\frac{C(n_0+1)}{2}   \bigg) }{n_0 } =0.
\]

Therefore, 
\[
 \mathbb{E}_1[T(b)]  \le 1 + n_0( 1 + o(n_0)) ,  \qquad  n_0 \rightarrow \infty, \qquad 
\]
where $n_0 \triangleq  \bigg\lfloor \frac{b+\mu}{I} \bigg\rfloor$. 
\end{proof}

\section{NUMERICAL EXPERIMENTS}
\label{sec:numerical}
In this section, we present numerical results that validate the effectiveness of our conditional score learning, as well as the score-based change detection method. We conduct two sets of experiments: One is based on synthetic Markov data path generated from a Gaussian kernel; the other is using high-dimensional real-world data from the CMU Motion Capture dataset~\cite{CMU_MoCap}. In both cases, we demonstrate that the difference of the conditional Hyvärinen score exhibits a negative or positive drift before and after the change, and it is effective for detecting the change. 

\subsection{Synthesize Markov Process with Gaussian Kernel}

\subsubsection{Discrete-Time Transition and Conditional Score}
We simulate the following process: 
\begin{equation}
\begin{split}
\label{eq:MarkovProcessFromSDE}
  X_{n+1} &= X_n + \,[f(X_n) - \nabla V(X_n)] + \sigma \, Z_n, \\
  Z_n &\sim \mathcal{N}(0,I).
  \end{split}
\end{equation}
This gives a discrete-time Markov process with a Gaussian transition kernel. Specifically, the conditional transition density is
\[
p(y \mid x)
= \mathcal{N}\!\Big(
y;\;
\underbrace{x + \,[f(x) - \nabla V(x)]}_{\mu(x)},
\;
\sigma^2 \, I
\Big).
\]
Note that, while the transition is Gaussian, the mean $\mu(x)$ is a nonlinear function of $x$ due to the inclusion of $f(x)$ and $\nabla V(x)$. For the motivation behind the structure of recursion in \eqref{eq:MarkovProcessFromSDE}, we refer to \cite{song2021scorebased}.

The conditional score can be explicitly derived in this case:
\begin{align}
 \nabla_y \log p(y\mid x)  = -\,\frac{1}{\sigma^2}\, 
\bigg(y - x - ( f(x) - \nabla V(x) )\bigg).  
\label{eq:ground_truth}
\end{align}
This expression serves as the ground-truth target for training a neural network to approximate the conditional score function $\nabla_y \log p(y\mid x)$. The ground truth is also used to identify the appropriate neural network architecture.  

\medskip
\subsubsection{Neural Network Architecture for Conditional Score Learning}
The synthetic Markov chain is simulated in a 10-dimensional space, i.e., each state  $X_n \in \mathbb{R}^{10}$. To learn the conditional score, we construct training pairs $(X_n, X_{n+1}) \in \mathbb{R}^{10} \times \mathbb{R}^{10},$ resulting in input vectors of dimension 20. We use a fully connected feedforward neural network to model the score. The network consists of three hidden layers, each with 128 units and SiLU (Sigmoid Linear Unit) activation functions. The output layer has 10 units, corresponding to the dimension of $X_n $, and produces the estimated score vector. The network is trained using the surrogate loss function defined in equation~\eqref{eq:surrogate_loss}.

\medskip
\subsubsection{Score Learning Accuracy}

To assess the accuracy of the learned conditional score function, we compare the network's output to the ground-truth gradient $\nabla_y \log p(y \mid x)$, whose closed form is in ~\eqref{eq:ground_truth}. We report the following metrics:

\begin{itemize}
    \item MSE: Mean squared error between the predicted and true score vectors.
    \item VarScale: The average squared norm of the true score, i.e., $\mathbb{E}[\|\nabla_y \log p(y \mid x)\|^2]$.
    \item RelError: The relative error, defined as $\text{MSE} / \text{VarScale}$.
\end{itemize}

\vspace{0.5em}
\noindent
The results for both the pre-change and post-change Markov processes are summarized below:

\begin{table}[h]
\centering
\begin{tabular}{lccc}
\toprule
Parameters of the Markov Processes & MSE & VarScale & RelError \\
\midrule
pre: $\alpha=0.3$, $\sigma=0.3$, shift $= 0.2$ & 4.44 & 223.0 & 1.99\% \\
post: $\alpha=0.6$, $\sigma=0.5$, shift $= 0.9$ & 2.88 & 80.2  & 3.59\% \\
\bottomrule
\end{tabular}
\caption{Score learning errors of two  Markov processes.}
\label{tab:score_accuracy}
\end{table}
\noindent
In both cases, the network achieves low relative error, indicating that the learned score closely approximates the true conditional gradient.

\medskip
\subsubsection{Change Detection Results}
We evaluate the behavior of the Hyvärinen score difference and SCUSUM statistic along a simulated Markov trajectory with a change point at  $n = 120 $. 
\begin{figure}[h]
    \centering
    \includegraphics[width=\linewidth]{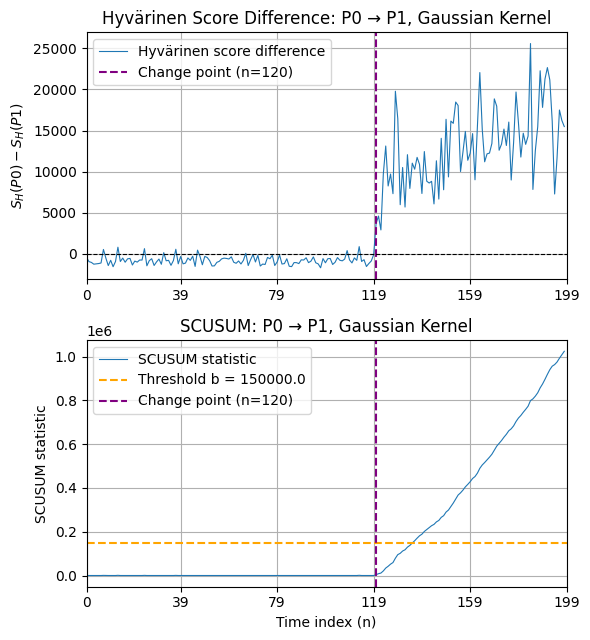}
    \caption{Top: Hyvärinen score difference. Bottom: Corresponding SCUSUM statistic.}
    \label{fig:qcd_gaussian}
\end{figure}

\noindent
The top panel of Figure~\ref{fig:qcd_gaussian} shows the Hyvärinen score difference
\[
s(X_n, X_{n-1}) = S_H(X_n, X_{n-1}, p_0) - S_H(X_n, X_{n-1}, p_1)
\]
over time. Before the change, the score difference exhibits a clear negative drift, remaining mostly below zero. After the change, it transitions to a positive drift, with the values increasing sharply and staying positive.

The bottom panel shows the corresponding SCUSUM statistic. The statistic remains close to zero before the change and rises steadily after the change occurs, eventually crossing the threshold $b = 1.5 \times 10^5$, and the detection delay is 16. This confirms that the learned conditional scores are effective for change detection, even in high-dimensional, nonlinear Markov processes.

\medskip
\subsubsection{Empirical False Alarm and Detection Delay Performance}
To validate the SCUSUM-based change detection method, we evaluate its performance under both the pre-change and post-change distributions. Specifically, we run numerical experiments to measure the average time to false alarm $\mathbb{E}_\infty[T_b]$ and the average detection delay $\mathbb{E}_1[T_b]$ for a range of thresholds $b$, using both non-truncated~\eqref{eq:no_truncation_alg} and truncated version~\eqref{eq:truncated_algorithm} of the test statistics. 
For the truncated algorithm, we use truncation level $M = 600$, and the corresponding estimated mean score difference $\mathbb{E}_\infty[\varphi(s(x, x')] = -\delta \approx -513.1 $. Theoretical bounds from Theorems~\ref{thm:false_alarm_lower_bound} and~\ref{thm:Delay_upper_bound} are also reported for reference, but they apply only to the truncated statistic.

\textit{Remark: } The theoretical bounds involve a parameter $\mu$ that characterizes the concentration scale of the test statistic under dependence as is defined in~\eqref{eq:mu_def}. 
For Markovian data, this parameter depends on the mixing constants $\tilde{l}$ and $\tilde{\lambda}$ in Doeblin's condition (see Assumption~\ref{thm:Doeblin’s Condition}), which are generally unknown and difficult to estimate in practice. 
To proceed, we empirically set the concentration parameter to $\mu_{\text{est}} = 2.05|\varphi|$. Although this estimate is heuristic, it provides a practical and numerically consistent approximation that aligns well with the empirical behavior observed in our experiments.

To measure the false alarm performance, a single long trajectory of one million steps is generated only using the pre-change Markov kernel, with every $X_n \sim P_0$, and based on the data path, we calculate the Hyvärinen score difference of each pair $(X_{n-1}, X_n)$. Each time a false alarm occurs (i.e., when the SCUSUM statistic exceeds the threshold $b$), the interval time between the two alarms is recorded and the detection statistic resets. Averaging over all the false alarm intervals yields an empirical estimate of $\mathbb{E}_\infty[T_b]$. 
Figure~\ref{fig:false_alarm_delay} shows the mean time to false alarm 
for both truncated (green) and non-truncated (blue) versions of the SCUSUM statistic, as well as the lower bound (orange) from Theorem~\ref{thm:false_alarm_lower_bound}, across a range of thresholds $b$.
Because truncation limits the maximum increment in the SCUSUM recursion, the test statistic grows slowly toward the threshold, resulting in longer false alarm intervals than the non-truncated one. It is important to note that the lower bound here applies only to the truncated algorithm, as defined in~\eqref{eq:truncated_algorithm}. Although the bound is conservative in magnitude, it captures the exponential growth trend with respect to the threshold $b$. 
\begin{figure}[htbp]
  \centering
  \includegraphics[width=0.9\linewidth]{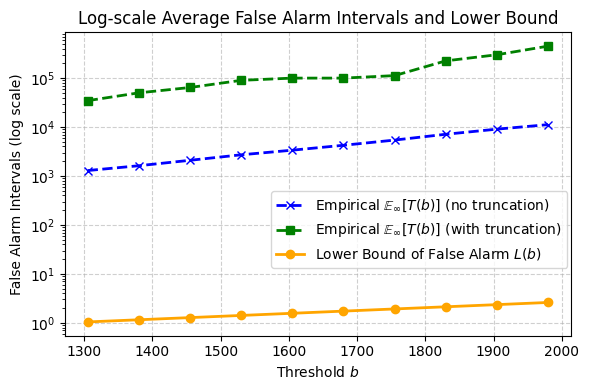}
  \caption{
    Log-scale average false alarm intervals and lower bound.}
  \label{fig:false_alarm_delay}
\end{figure}

To assess detection delay, we generate a post-change trajectory of 10,000 steps using the Markov kernel $P_1$, with $X_n \sim P_1$ for all $n$. Each time the SCUSUM statistic exceeds the threshold $b$, the delay is recorded and the detection statistics are reset. The average of these intervals provides an empirical estimate of $\mathbb{E}_1[T_b]$. As shown in Figure~\ref{fig:detection_delay}, 
the empirical detection delay is plotted for both the truncated (green) and untruncated (blue) SCUSUM statistics, along with the theoretical upper bound derived in Theorem~\ref{thm:Delay_upper_bound}. The bound applies only to the truncated setting, where the SCUSUM increments are clipped as in~\eqref{eq:truncated_algorithm}; no theoretical guarantee is available for the non-truncated case. Because the increments in the SCUSUM recursion are capped, the truncated version grows slowly and consequently shows slightly longer detection delays. The empirical detection delay remains below the theoretical upper bound, which grows linearly with threshold $b$, when $b$ is large, thereby confirming the validity of Theorem~\ref{thm:Delay_upper_bound} for the truncated statistic.

\begin{figure}[htbp]
  \centering
  \includegraphics[width=0.9\linewidth]{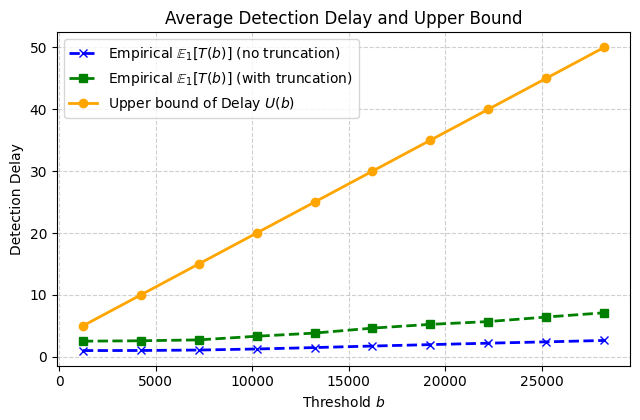}
  \caption{ Detection delay and upper bound.}
  \label{fig:detection_delay}
\end{figure}

\subsection{Real-World Experiments on CMU Motion Capture Data}

\subsubsection{Dataset}
The CMU Motion Capture (MoCap) database contains 3D human movement recordings collected by markers, which are carefully placed to get maximal hinge joint information.   Subjects wear markers, and a set of high-speed cameras tracks the 3D trajectories of the markers. From these measurements, a skeletal model is fit, representing the human body as an articulated joint (such as hips, knees, ankles, shoulders, elbows, wrists, and spine segments) and rigid segments connecting them. At each time frame, the system records a high‑dimensional state vector composed of the joints' angles of this skeleton. The dataset consists of thousands of distinct motion trials covering a wide variety of activities—walking, running, jumping, dancing, sports interactions, and so on.  Each trial represents a time sequence of skeleton configurations at around 120 Hz. The dimensionality of each frame depends on the skeleton model used,  usually around  90–100 joint‐angle features per frame. Each trial typically spans several seconds to tens of seconds, yielding sequences of hundreds to thousands of frames. 

\begin{figure}[h]
    \centering
    \includegraphics[height=4cm]{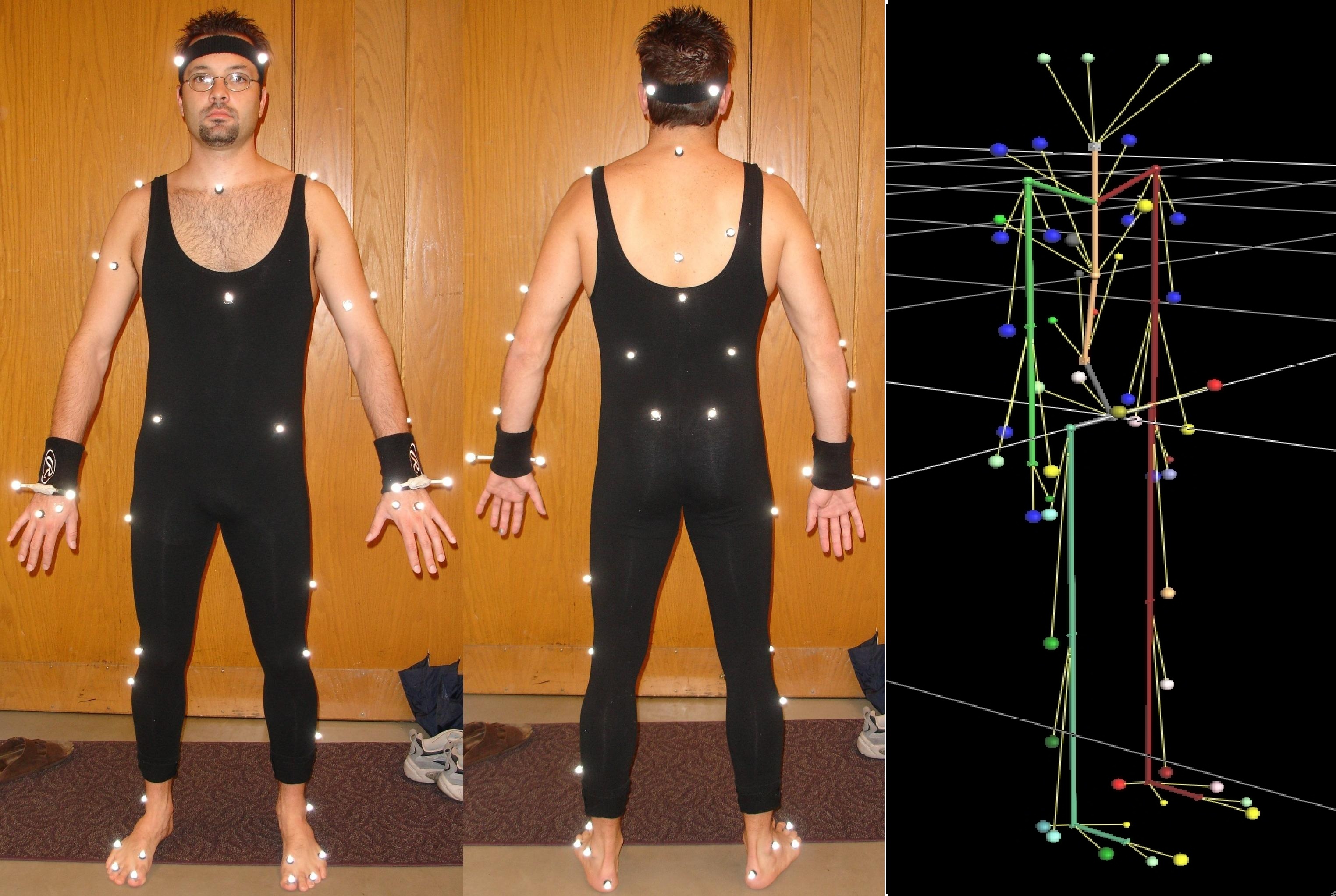}  
    \caption{Marker setup: front and back view. The balls represent markers; the thick colored segments represent bones~\cite{CMU_MoCap}. }
    \label{fig:marker}
\end{figure}

Figure~\ref{fig:marker} shows the marker placements used in the motion capture system, with front and back views of a subject wearing the reflective markers.

\begin{figure*}[h]
    \centering
    \subfloat[Sample frames from a video on running activity]{
\includegraphics[width=0.19\textwidth,height=0.11\textwidth]{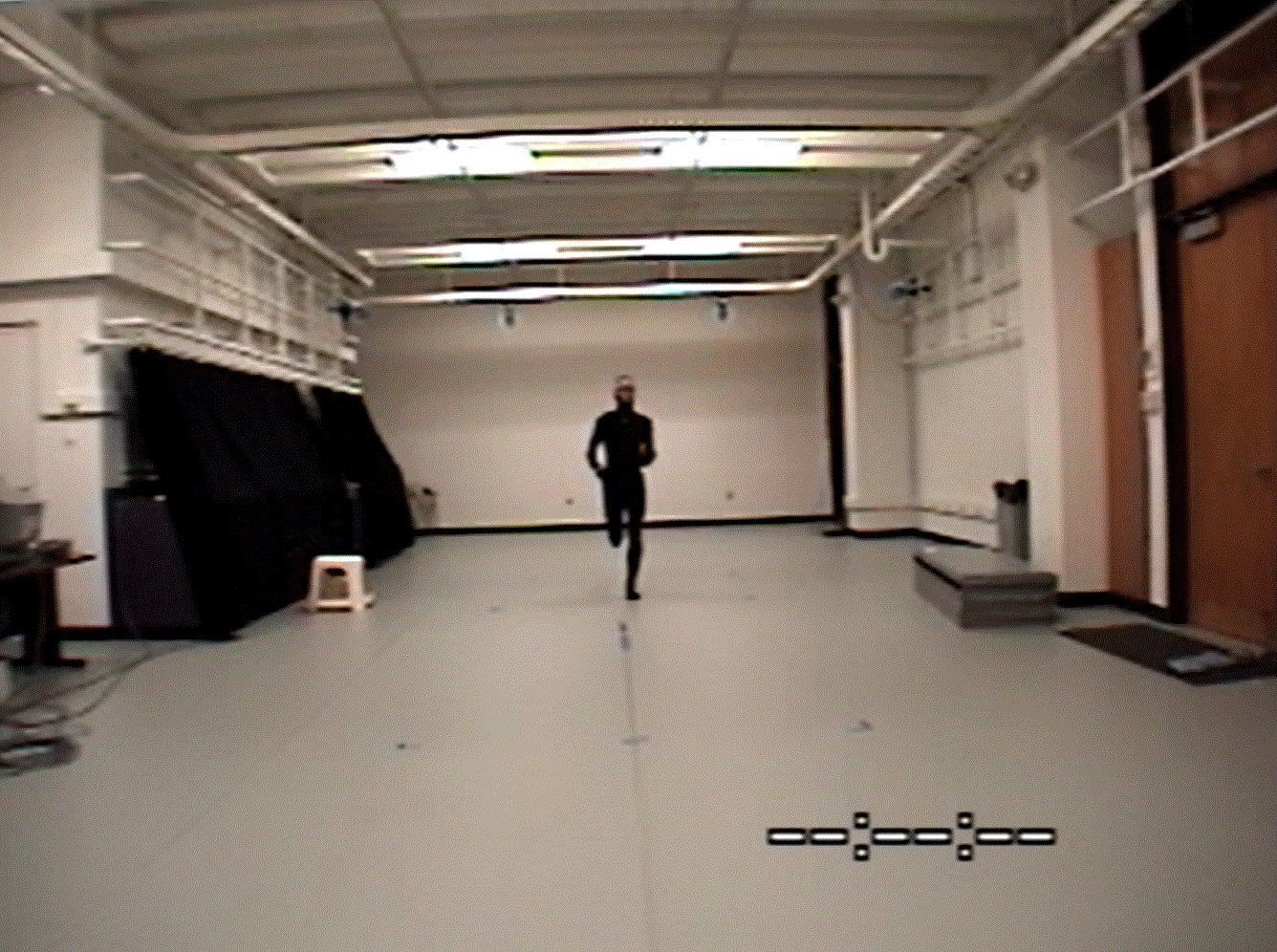}
    \hspace{1mm}
\includegraphics[width=0.19\textwidth,height=0.11\textwidth]{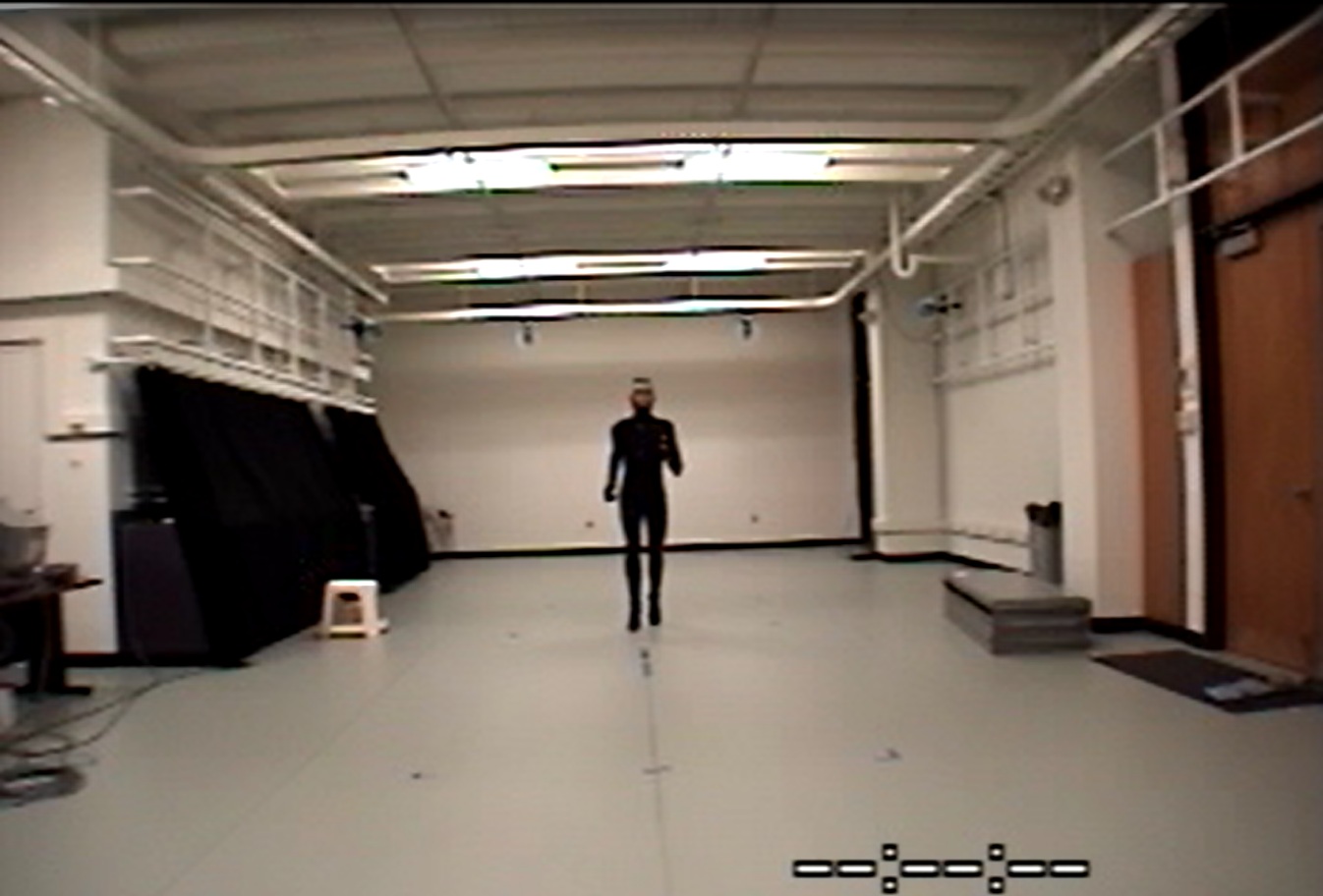}
    \hspace{1mm}    \includegraphics[width=0.19\textwidth,height=0.11\textwidth]{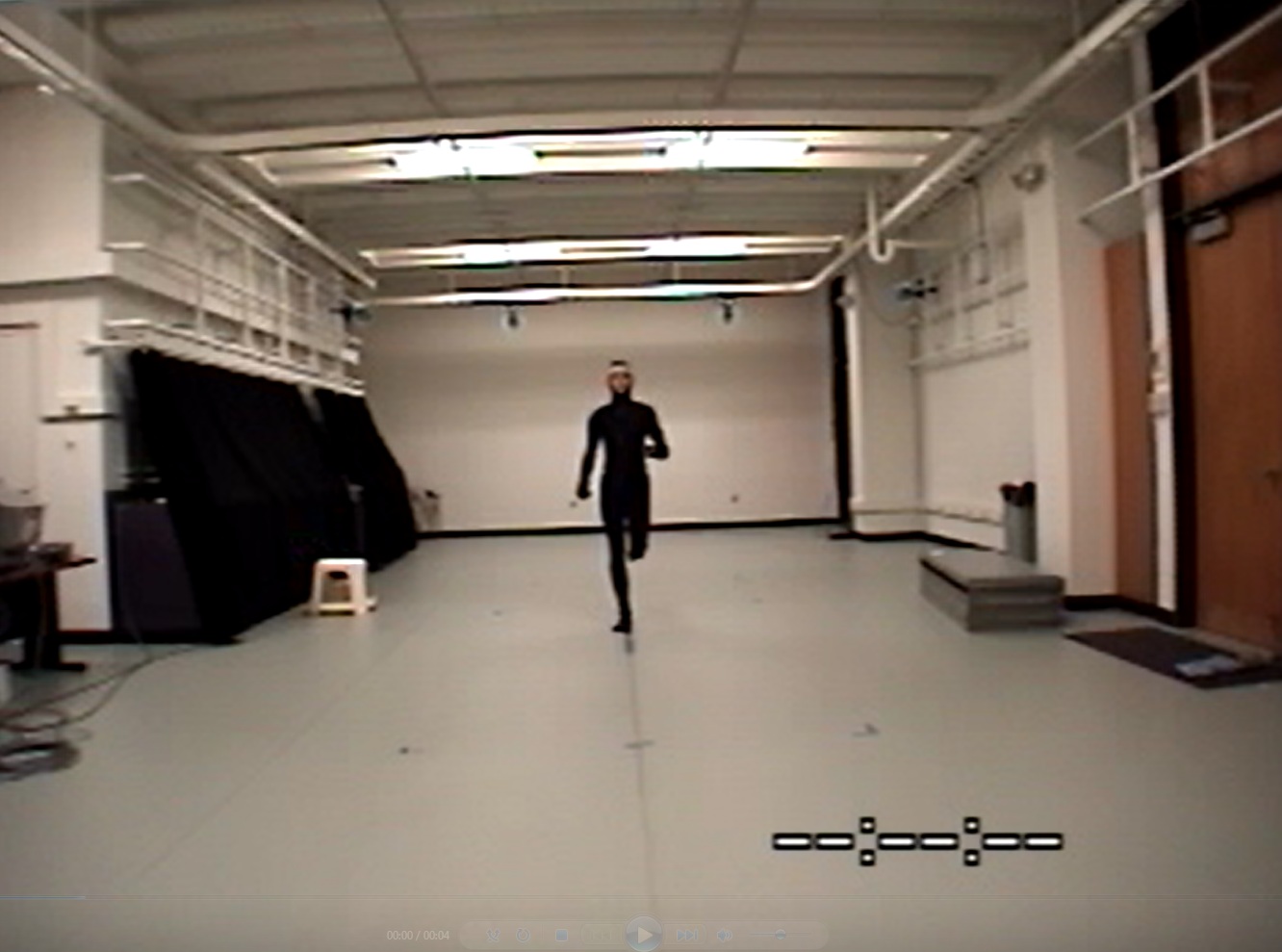}
     \hspace{1mm}
\includegraphics[width=0.19\textwidth,height=0.11\textwidth]{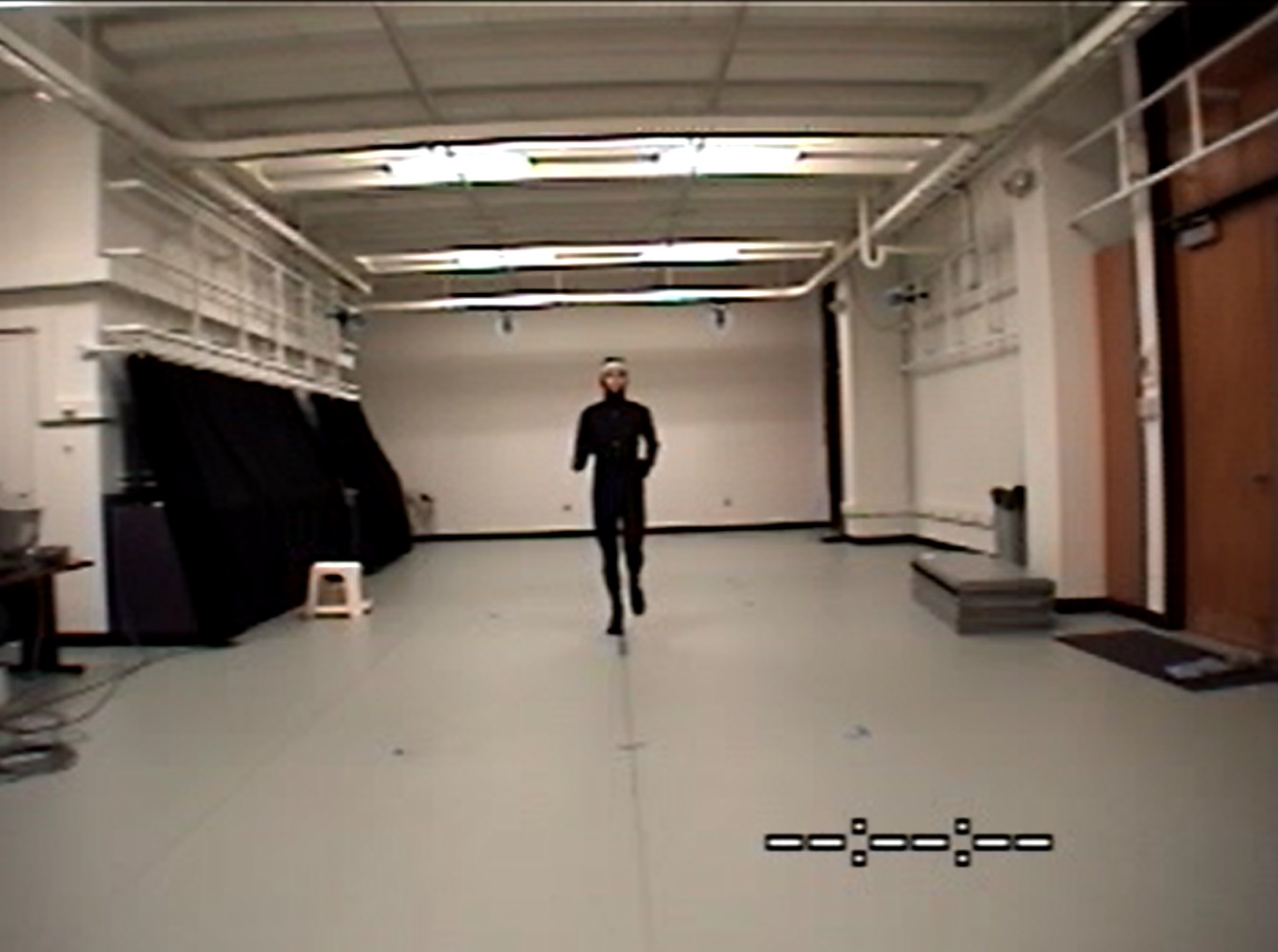}  
    }
    \vspace{1.5mm}

    \subfloat[Sample frames from a video on basketball activity]{
\includegraphics[width=0.19\textwidth,height=0.11\textwidth]{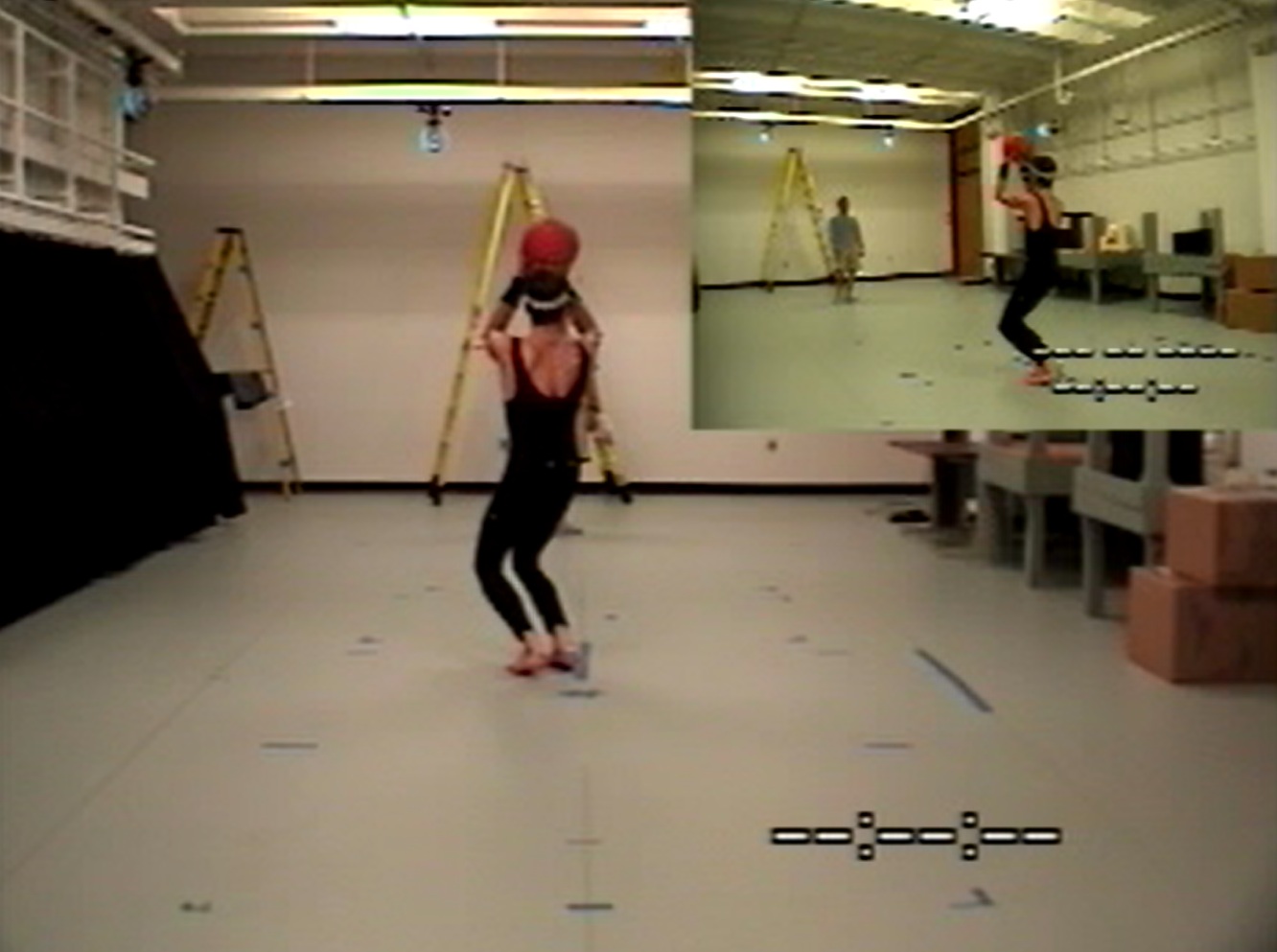}
    \hspace{1mm}
\includegraphics[width=0.19\textwidth,height=0.11\textwidth]{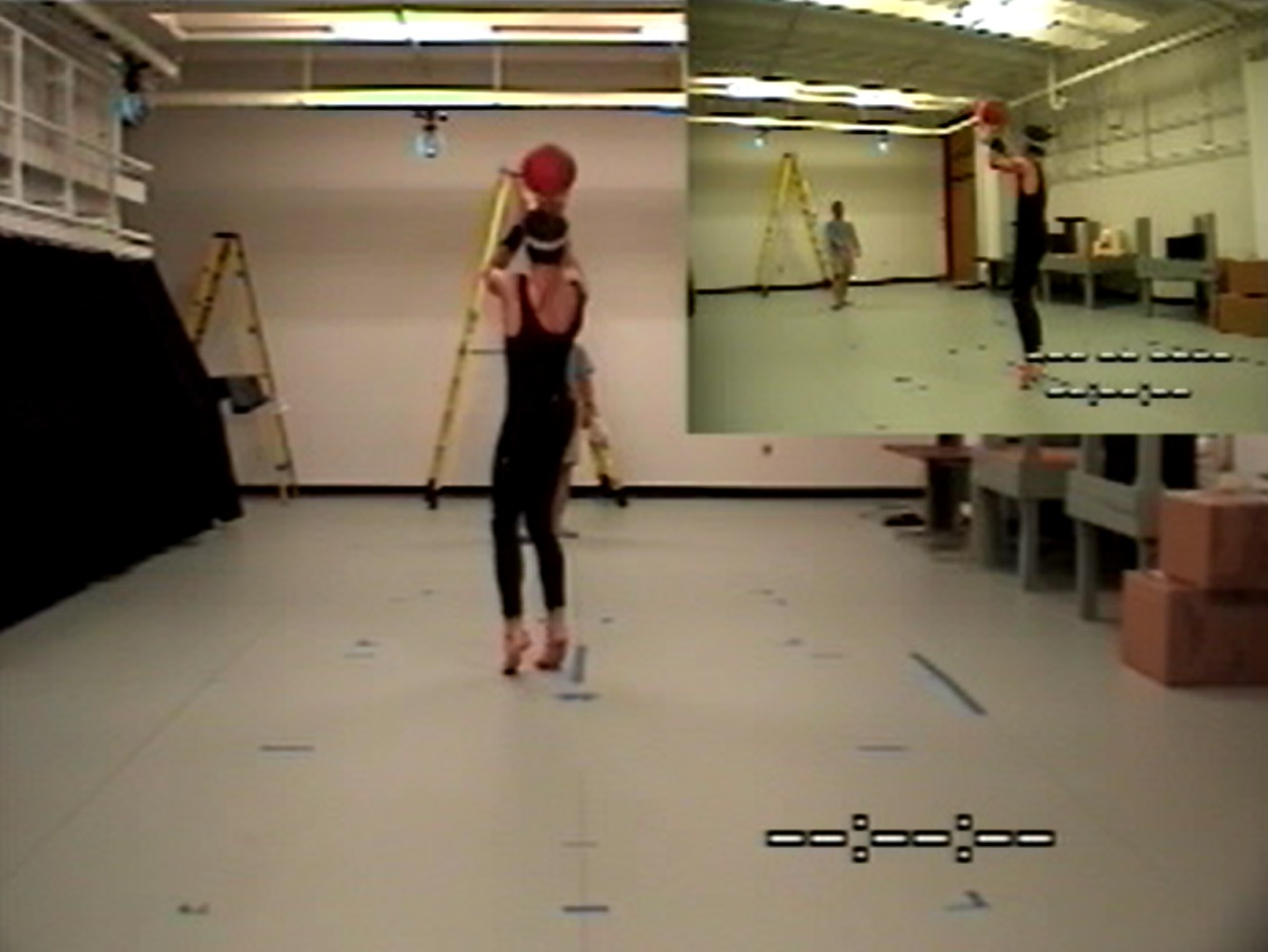}
    \hspace{1mm}    \includegraphics[width=0.19\textwidth,height=0.11\textwidth]{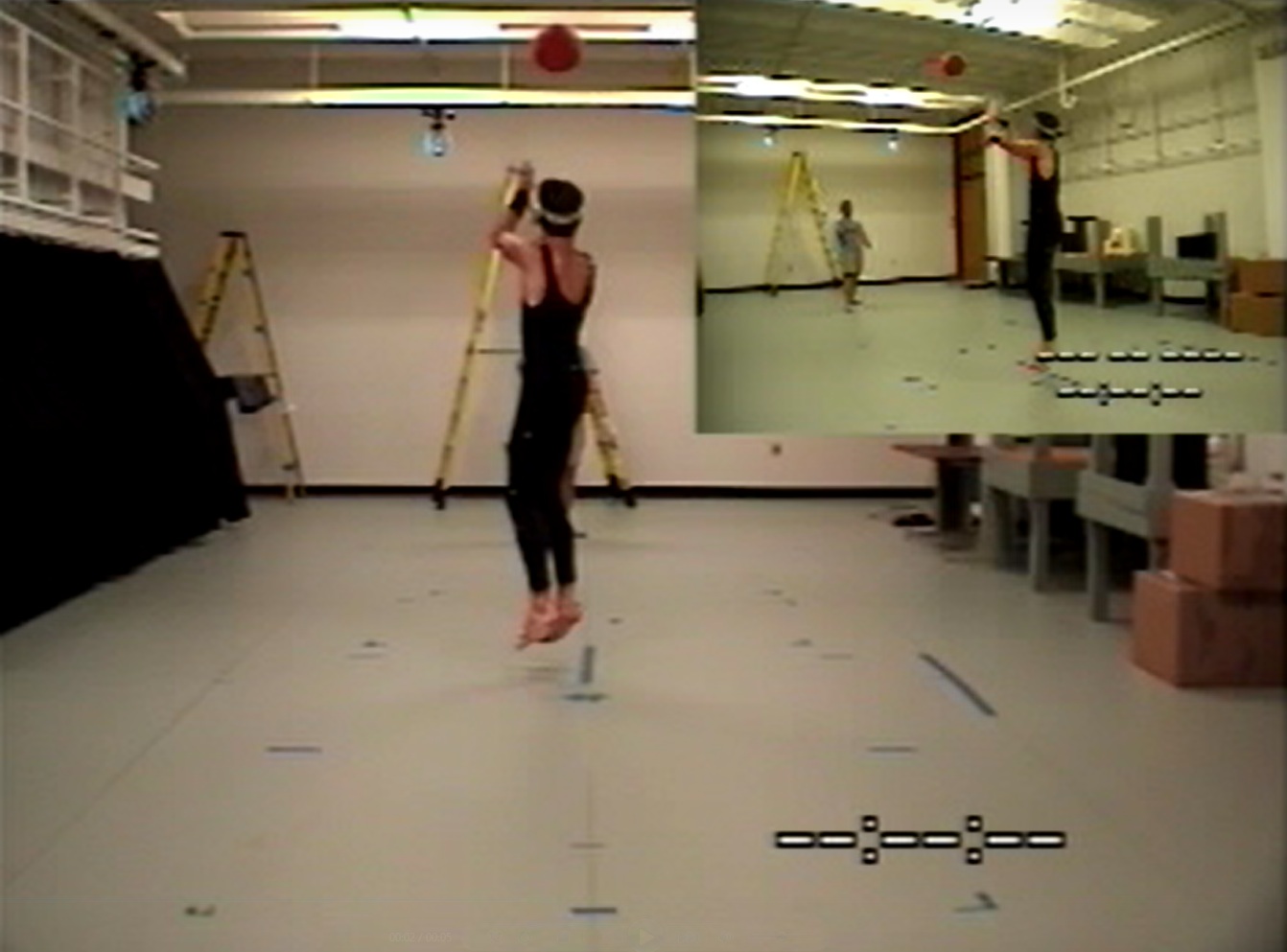}
     \hspace{1mm}
\includegraphics[width=0.19\textwidth,height=0.11\textwidth]{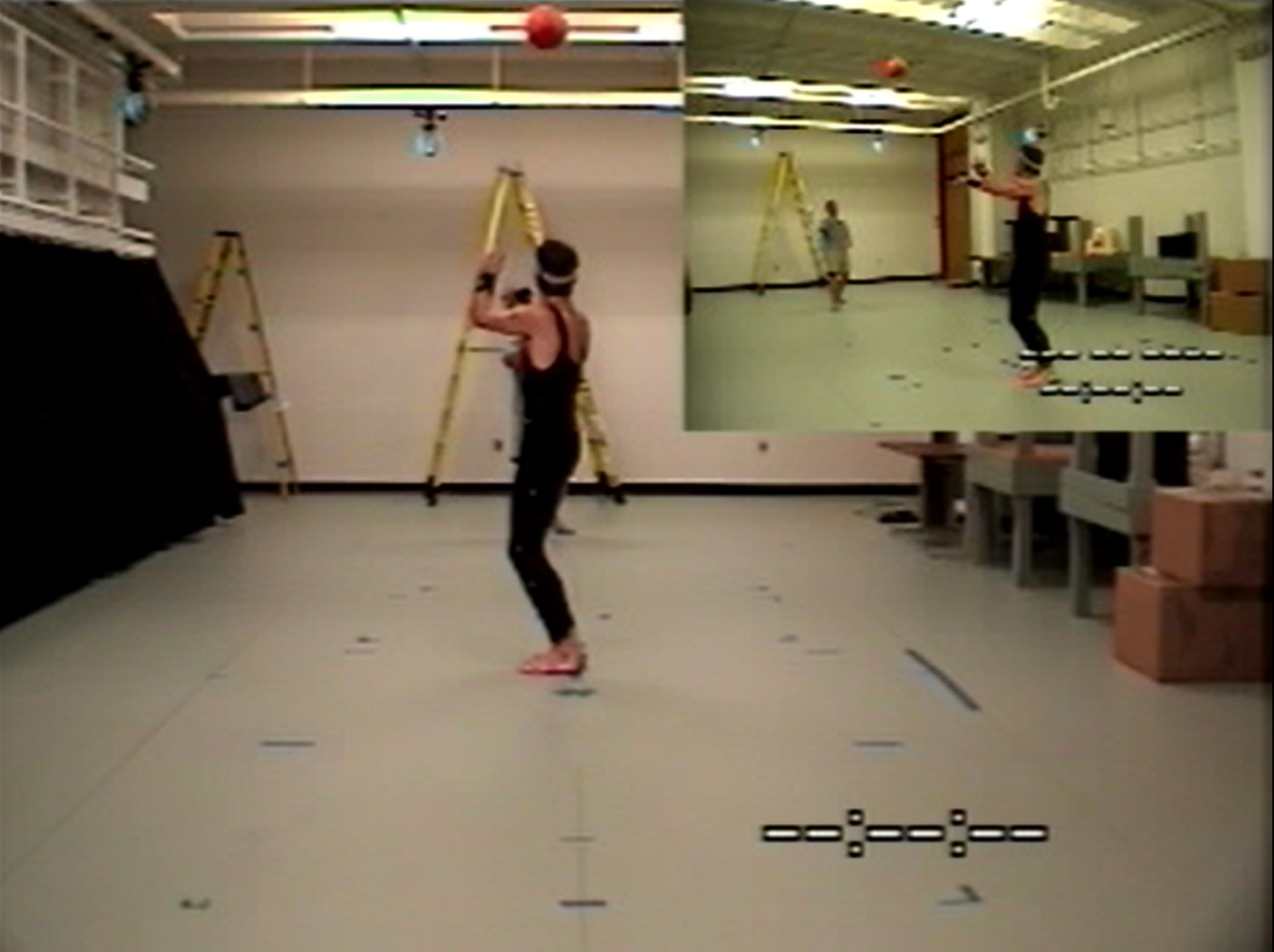}  
    }
    \vspace{1.5mm}

    \subfloat[Sample frames from a video on jumping activity]{
\includegraphics[width=0.19\textwidth,height=0.11\textwidth]{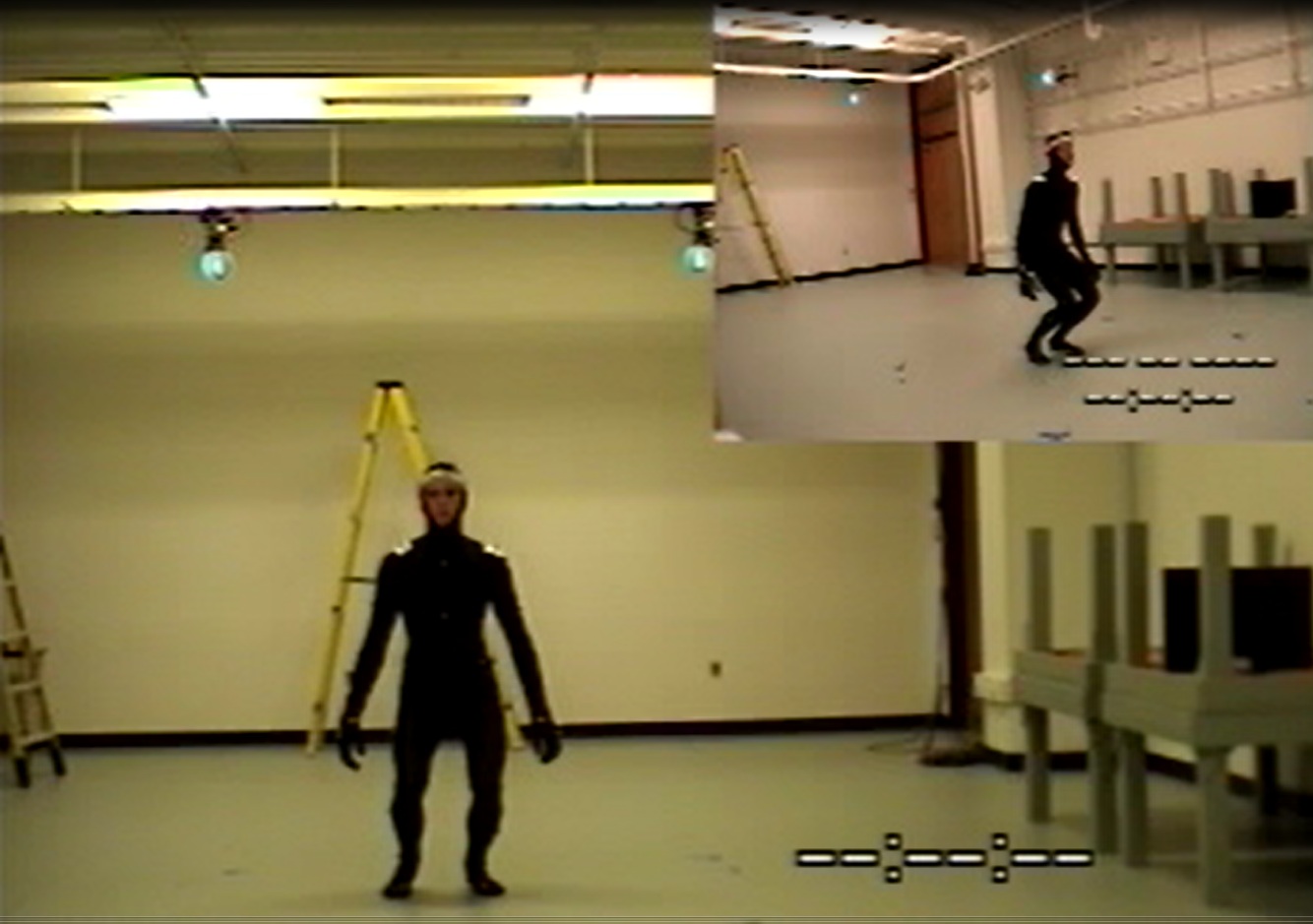}
    \hspace{1mm}
\includegraphics[width=0.19\textwidth,height=0.11\textwidth]{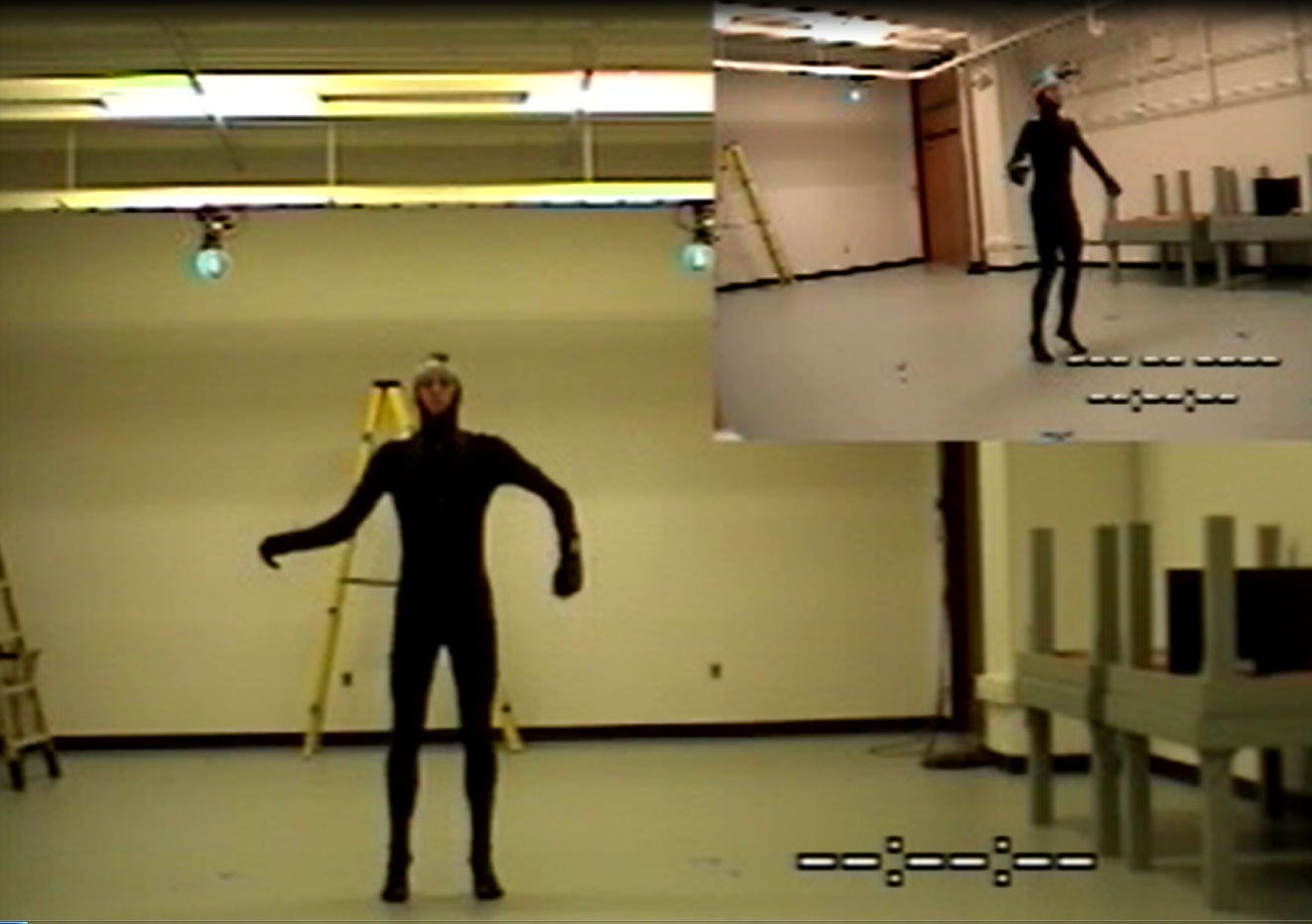}
    \hspace{1mm}    \includegraphics[width=0.19\textwidth,height=0.11\textwidth]{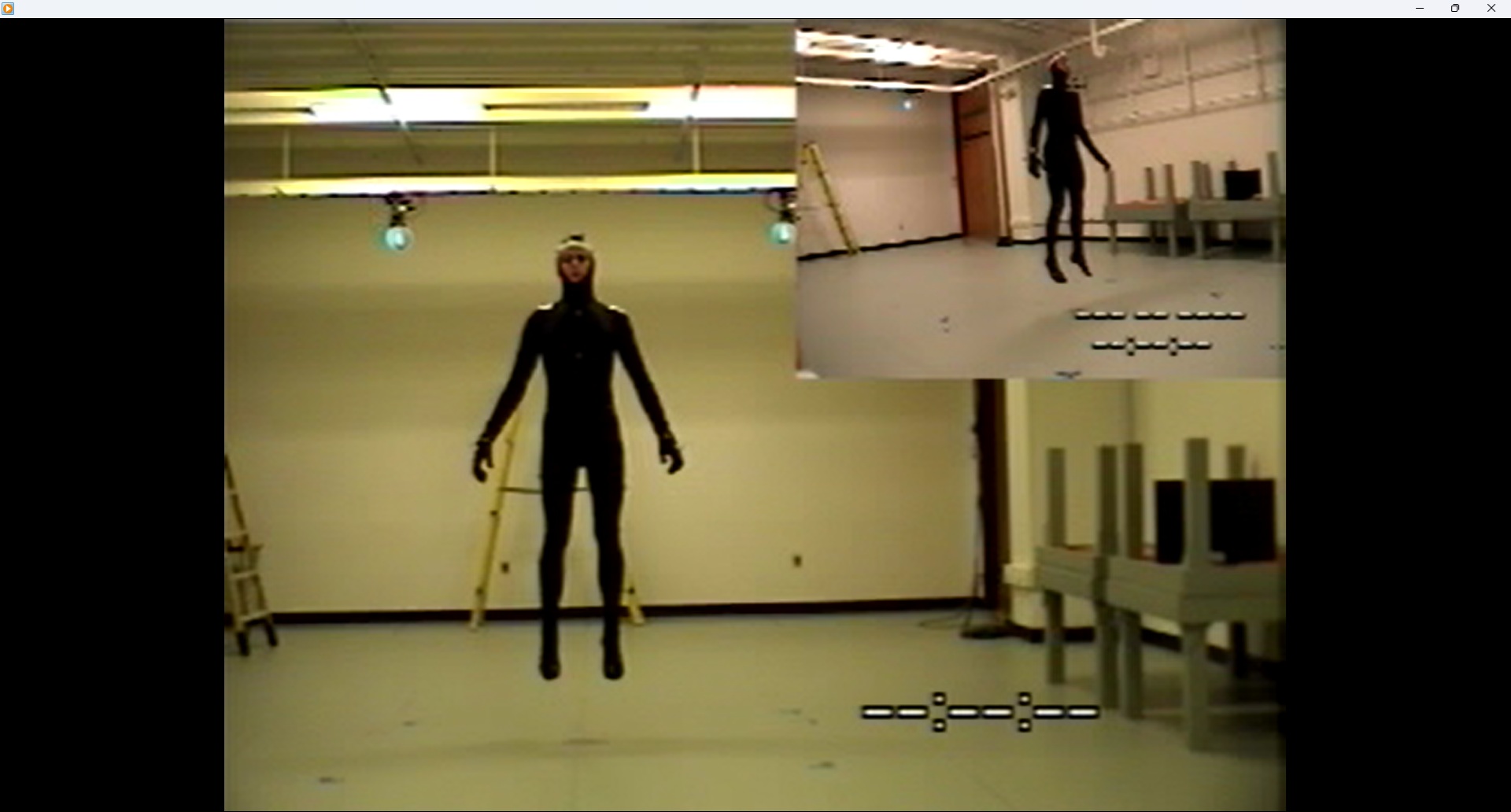}
     \hspace{1mm}
\includegraphics[width=0.19\textwidth,height=0.11\textwidth]{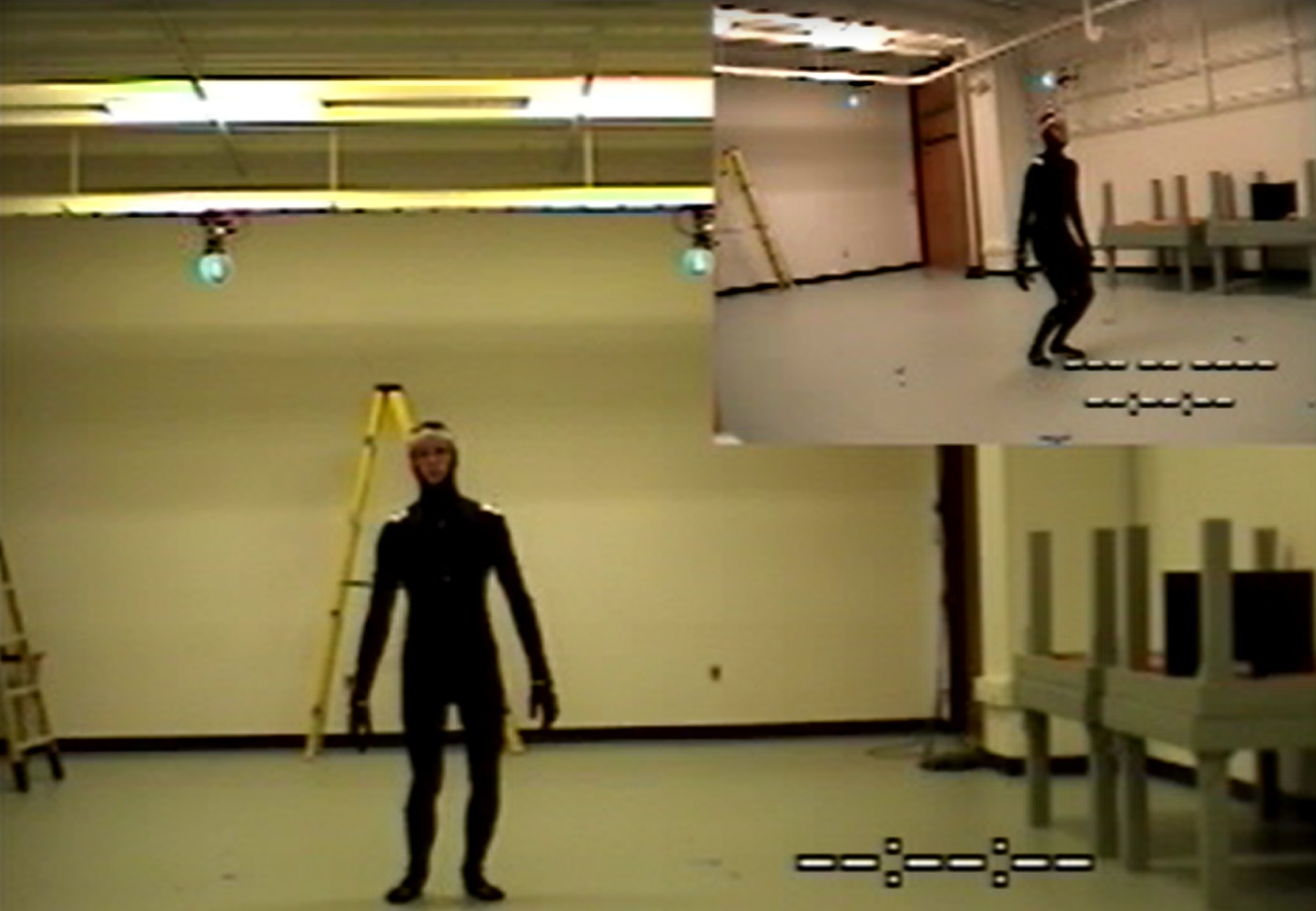}  
    }

    \caption{CMU motion capture~\cite{CMU_MoCap} sequences under three scenarios: running (top), basketball (middle), and jumping (bottom). Each row presents a sequence of four frames captured at random but increasing time indices. }
    \label{fig:cmu_three_scenarios}
\end{figure*}
\begin{figure*}[h]
    \centering
    \begin{subfigure}{0.32\textwidth}
        \includegraphics[width=\textwidth]{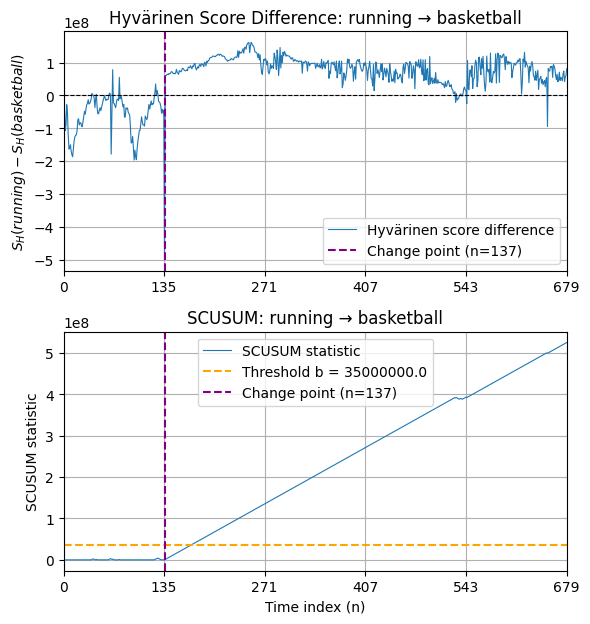}
        \caption{{\centering
        Running $\rightarrow$ Basketball \\ 
        \qquad Detect change at $n = 171$, delay = 34}}
    \end{subfigure}
    \hfill
    \begin{subfigure}{0.32\textwidth}
        \includegraphics[width=\textwidth]{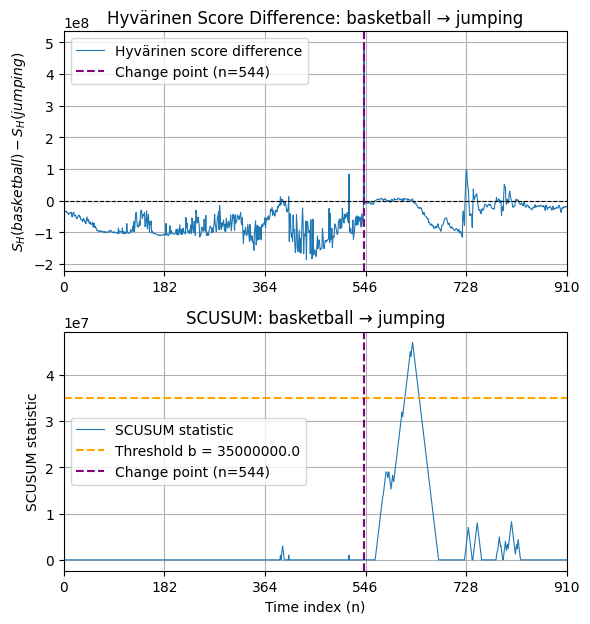}
        \caption{{\centering
        Basketball $\rightarrow$ Jumping \\ 
        \qquad  Detect change at $n = 618$, delay = 74}}
    \end{subfigure}
    \hfill
    \begin{subfigure}{0.32\textwidth}
        \includegraphics[width=\textwidth]{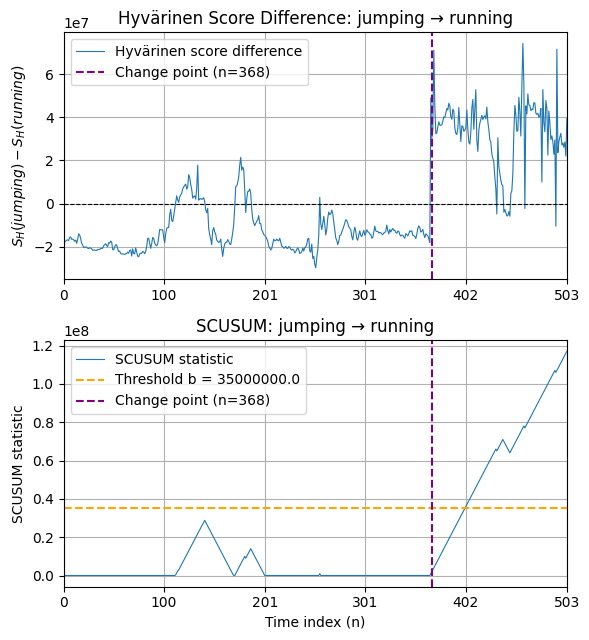}
        \caption{{\centering
       Jumping  $\rightarrow$ Running  \\ 
       \qquad  Detect change at $n = 401$, delay = 33}}
    \end{subfigure}
    \caption{The instantaneous Hyvärinen score difference and the respective CUSUM statistics of three motion transition scenarios, we use a universal threshold $b = 3.5 \times 10^7$ for all the scenarios.}
    \label{fig:cmu_motion_results}
\end{figure*}
Fig.~\ref{fig:cmu_three_scenarios} illustrates three representative activities of the CMU motion capture data used in our experiments. 
It shows sample motion sequences from three distinct activities: 
\textit{running}, \textit{basketball}, and \textit{jumping}. 
Each row depicts four frames, collected at four distinct time points representing steps in the motion. These frames illustrate the dynamic progression of each activity, capturing how the subject's pose evolves over time.
These sequences are converted into high-dimensional state vectors of joint angles in $ \mathbb{R}^{93}$, where each vector encodes the full set of joint angles for a single frame. To construct Markov transitions $(x_{n-1}, x_n)$, we pair each consecutive frame in time, thereby treating each motion sequence as a discrete-time trajectory of evolving states. These transitions are used to train the conditional score network and evaluate the Hyvärinen score differences for change detection.

\subsubsection{Training with neural network}
To model the conditional score function, we use a fully connected neural network. The input to the network is the concatenated vector $(\mathbf{y}, \mathbf{x}) \in \mathbb{R}^{186}$, where both $\mathbf{x}$ and $\mathbf{y}$ are state vectors in $\mathbb{R}^{93}$. The output is a vector in $\mathbb{R}^{93}$, corresponding to the estimated conditional score $\nabla_{\mathbf{y}} \log p(\mathbf{y} \mid \mathbf{x})$.

The architecture consists of 5 fully connected layers. Each hidden layer has width 512 and is followed by a SiLU activation, while the final output layer is linear. For each activity transition, a dedicated model is trained using only the corresponding segment of motion data.

\subsubsection{Detecting changes in the type of activity}

We evaluate the detection performance of our method on three motion transition scenarios from the CMU MoCap dataset: \textit{running} $\to$ \textit{basketball}, \textit{basketball} $\to$ \textit{jumping}, and \textit{running} $\to$ \textit{jumping}. Fig.~\ref{fig:cmu_motion_results} shows the Hyvärinen score difference (top row) and the corresponding SCUSUM statistic (bottom row) for each case. In all scenarios, the Hyvärinen score exhibits a clear drift after the change point, and the SCUSUM statistic successfully accumulates this signal and crosses the threshold shortly after the true change, demonstrating effective and timely detection.

In the \textit{running $\to$ basketball} and \textit{running $\to$ jumping} scenarios, we observe the expected drift behavior: the Hyvärinen score difference shows a clear negative drift before the change and a positive drift afterward, resulting in a rapid and reliable detection. 

In contrast, the \textit{basketball $\to$ jumping} transition yields only a weak, slightly negative empirical drift under the post-change data, reflecting the high similarity between the two underlying transition kernels. Nevertheless, the distribution of the score differences changes noticeably after the true change point, allowing the cumulative SCUSUM statistic to cross the threshold and raise an alarm. 

\section{Conclusion}
\label{sec:conclusion}
We propose a framework for conditional score learning and apply it to the quickest change detection problem in Markov processes.  We design a conditional score-based CUSUM procedure and introduce a truncated version of it for data stability, via estimating the conditional score with neural networks. 
Using Hoeffding's inequality for Uniform and ergodic Markov processes, we give exponential lower bounds on mean time to false alarm and asymptotic upper bounds on detection delays. 
All these results show that the score-based methods can be extended from i.i.d. settings to dependent settings with rigorous performance guarantees.

\section*{Acknowledgment}
This work was supported by the U.S. National Science Foundation under awards numbers 2334898 and 2334897. 

\ifCLASSOPTIONcaptionsoff
  \newpage
\fi


\balance
\bibliographystyle{IEEEtran}
\bibliography{bibtex/bib/IEEEabrv,bibtex/bib/IEEEexample}

\end{document}